\documentclass{article}
\usepackage{fullpage}

\usepackage{graphicx}  
\usepackage{amssymb}
\usepackage{amsmath}  
\usepackage{amstext}  
\usepackage{amsthm}   
\usepackage{natbib}
\usepackage[flushleft]{threeparttable}  
\usepackage{float}  
\usepackage{hyperref}  

\usepackage{algorithm,algorithmic}
\newtheorem{thm}{Theorem}

\newtheorem{lemma}{Lemma}

\newtheorem{ass}{Assumption}

\hypersetup{colorlinks = true, citecolor = blue}  
\usepackage{subfigure}
\usepackage{xcolor}
\usepackage{booktabs} 
\usepackage{caption}

\def \E {\mathrm{E}}
\def \x {\mathbf{x}}
\def \g {\mathbf{g}}

\def \D {\mathbf{D}}

\def \u {\mathbf{u}}

\def \w {\mathbf{w}}
\def \R {\mathbb{R}}

\def \A {\mathcal{A}}
\def \B {\mathcal{B}}
\def \q {\mathbf{q}}

\def \p {\mathbf{p}}
\def \q {\mathbf{q}}

\def \h {\mathbf{h}}

\def \X {\mathcal{X}}

\def \P {\mathcal{P}}
\def \Q {\mathcal{Q}}

\def \w {\mathbf{w}}
\def \E {\mathrm{E}}

\def \g {\mathbf{g}}
\def \R {\mathbb{R}}
\def \u {\mathbf{u}}

\def \x {\mathbf{x}}
\def \y {\mathbf{y}}
\def \yh {\widehat{\y}}
\def \X {\mathcal{X}}
\def \Y {\mathcal{Y}}

\def \Alg {Dash}

\usepackage{authblk}
\begin{document}
\title{\bf Dash: Semi-Supervised Learning with Dynamic Thresholding}
\author[1]{Yi Xu \thanks{yixu@alibaba-inc.com}}
\author[1]{Lei Shang \thanks{sl172005@alibaba-inc.com}}
\author[1]{Jinxing Ye \thanks{zhengze.yjx@alibaba-inc.com}}
\author[1]{Qi Qian \thanks{qi.qian@alibaba-inc.com}}
\author[2]{Yu-Feng Li \thanks{liyf@nju.edu.cn}}
\author[1]{Baigui Sun \thanks{baigui.sbg@alibaba-inc.com}}
\author[1]{Hao Li \thanks{lihao.lh@alibaba-inc.com}}
\author[1]{Rong Jin \thanks{jinrong.jr@alibaba-inc.com}}
\affil[1]{Machine Intelligence Technology, Alibaba Group}
\affil[2]{National Key Laboratory for Novel Software Technology, Nanjing University}
\date{\today}
\maketitle

\begin{abstract}
   While semi-supervised learning (SSL) has received tremendous attentions in many machine learning tasks due to its successful use of unlabeled data, existing SSL algorithms use either all unlabeled examples or the unlabeled examples with a fixed high-confidence prediction during the training progress. However, it is possible that too many correct/wrong pseudo labeled examples are eliminated/selected. In this work we develop a simple yet powerful framework, whose key idea is to select a subset of training examples from the unlabeled data when performing existing SSL methods so that only the unlabeled examples with pseudo labels related to the labeled data will be used to train models. The selection is performed at each updating iteration by only keeping the examples whose losses are smaller than a given threshold that is dynamically adjusted through the iteration. Our proposed approach, Dash, enjoys its adaptivity in terms of unlabeled data selection and its theoretical guarantee. Specifically, we theoretically establish the convergence rate of Dash from the view of non-convex optimization. Finally, we empirically demonstrate the effectiveness of the proposed method in comparison with state-of-the-art over benchmarks. 
\end{abstract}

\section{Introduction}
In spite of successful use in a variety of classification and regression tasks, supervised learning requires large amount of \emph{labeled} training data. In many machine learning applications, labeled data can be significantly more costly, time-consuming and difficult to obtain than the \emph{unlabeled} data~\citep{zhu2005semi}, since they usually require experienced human labors from experts (e.g., a doctor in detection of covid-19 by using X-ray images). Typically, only a small amount of labeled data is available, but there is a huge amount of data without label. This is one of key hurdles in the development and deployment of machine learning models. 

Semi-supervised learning (SSL) is designed to improve learning performance by leveraging an abundance of unlabeled data along with limited labeled data~\citep{chapelle2006semi}. In much recent work, SSL can be categorized into several main classes in terms of the use of unlabeled data: consistency regularization, pseudo labeling, generic regularization (e.g., large margin regularization, Laplacian regularization, etc~\citep{chapelle2006semi}), and their combinations. With the image data augmentation technique, consistency regularization uses unlabeled data~\citep{baird1992document,schmidhuber2015deep} based on the condition that the model predictions between different perturbed versions of the same image are similar. Another line of work is to produce artificial label for unlabeled data based on prediction model and add them to 
the training data set. With different approaches of artificial label production, varies of SSL methods have been proposed in the literature including self-training~\citep{yarowsky1995unsupervised, lee2013pseudo,rosenberg2005semi, sajjadi2016regularization, laine2017temporal, xie2020self} and co-training~~\citep{blum1998combining, zhou2005tri, sindhwani2008rkhs, wang2008random, yu2008bayesian, wang2010new,  chen2011automatic}. Due to its capability to handle both labeled data and unlabeled data, SSL has been widely studied in diverse machine learning tasks such as image classification~\citep{sajjadi2016regularization,laine2017temporal, tarvainen2017mean, xie2020unsupervised, berthelot2019mixmatch, berthelot2019remixmatch}, natural language processing~\citep{turian2010word}, speech recognition~\citep{yu2010active}, and object detection~\citep{misra2015watch}. 

\begin{figure*}[t]
    \centering
    \subfigure[Number of selected unlabeled examples with correct pseudo labels]{\includegraphics[width=0.45\textwidth]{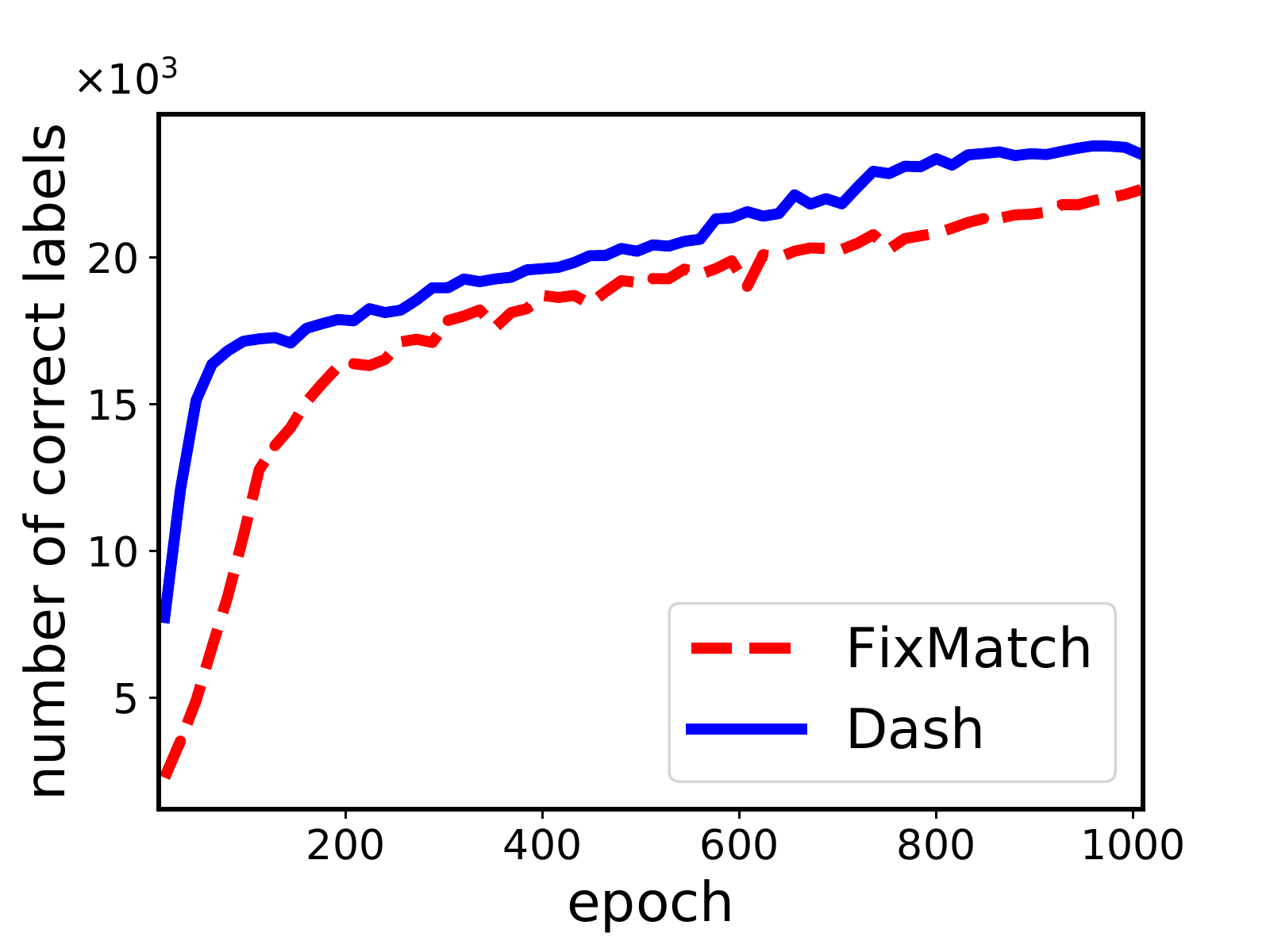}}
    \subfigure[Number of selected unlabeled examples with wrong pseudo labels]{\includegraphics[width=0.45\textwidth]{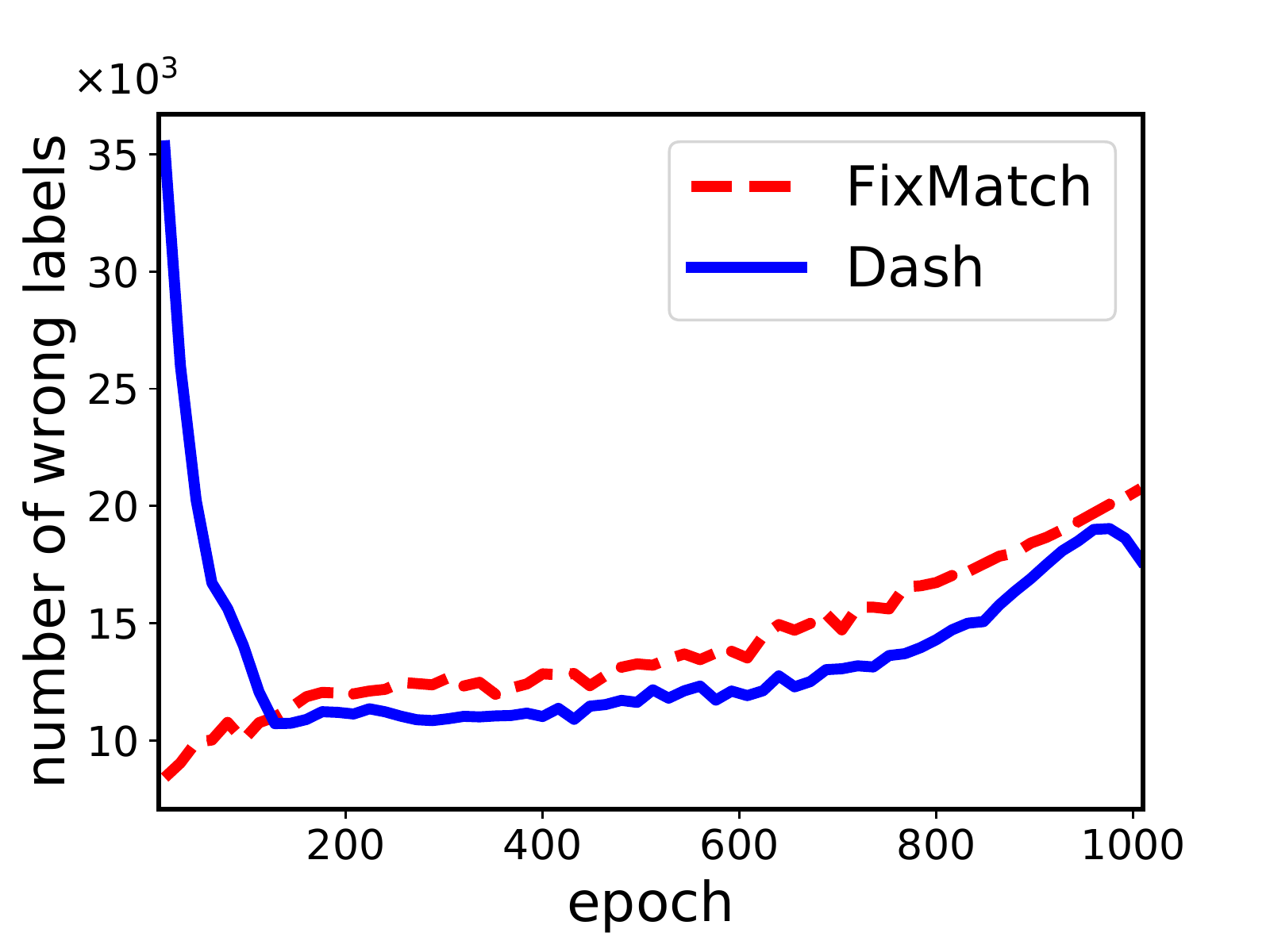}}
    \caption{An example of experimental results on Wide ResNet-28-8 for CIFAR-100 with 400 labeled images illustrates the reason of dynamically selecting unlabeled data to train learning models. Pseudo labels are generated based on the prediction models. FixMatch selects unlabeled example if its confidence prediction is greater than $0.95$, while the proposed Dash algorithm selects unlabeled example based on a dynamic threshold through optimization iterations. (a) The proposed Dash selects more examples with correct pseudo labels than that of FixMatch. (b) The proposed Dash maintains much more examples with wrong pseudo labels at the beginning but it will drop off more examples with wrong pseudo labels after several epochs, comparing to FixMatch.}\label{fig:dash}
\end{figure*}
Numerous empirical evidences show that unlabeled data in SSL can help to improve the learning performance~\citep{berthelot2019mixmatch,berthelot2019remixmatch,sohn2020fixmatch}, however, a series of theoretical studies~\citep{ben2008does,singh2009unlabeled,li2011towards, balcan2005pac} have demonstrated that this success is highly relying on a necessary condition that labeled data and unlabeled data with pseudo label come from the same distribution during the training process~\citep{zhu2005semi,van2020survey}.
Unfortunately, it has been shown that this condition does not always hold in real applications and thus it could hurt the performance~\citep{li2017learning,oliver2018realistic}. For example, the pseudo label of an unlabeled example generated by conventional SSL methods during the training progress is not correct~\citep{hataya2019unifying, li2020dividemix}. In this case, the degradation of model performance has been observed when using unlabeled data compared to the simple supervised learning model not using any unlabeled data at all~\citep{chapelle2006semi,oliver2018realistic}. Thus, {not all unlabeled data are needed in SSL}.

To improve SSL performance, multiple studies~\citep{guosafe2020,ren2020not} examined the strategies of weighting different training unlabeled examples by solving a bi-level optimization problem. It is also a popular idea to select a subset of training examples from unlabeled examples for SSL. For example, FixMatch~\citep{sohn2020fixmatch} uses the unlabeled examples with a fixed high-confidence prediction (e.g., $0.95$) in classification tasks using cross entropy loss. However, the fixed threshold may lead to eliminate too many unlabeled examples with correct pseudo labels (see Figure~\ref{fig:dash} (a)) and may lead to select too many unlabeled examples with wrong pseudo labels (see Figure~\ref{fig:dash} (b)). That is to say, the fixed  threshold is possible not good enough during the training progress and thus it could degrade the overall performance.

Unlike the previous work, we aim to the proposed approach enjoys its adaptivity in terms of unlabeled data selection and its theoretical guarantee. This inspires us to consider answering the following question in this study. \begin{center}
    {\bf Can we design a provable SSL algorithm that selects unlabeled data with dynamic thresholding?}
\end{center}

To this end, we propose a generic SSL algorithm with {\bf D}yn{\bf a}mic Thre{\bf sh}olding (\Alg) that can dynamically select unlabeled data during the training process. Specifically, \Alg~firstly runs over labeled data and obtains a threshold for unlabeled data selection. It then selects the unlabeled data whose loss values are smaller than the threshold to the training data-set. The value of threshold is gradually decreased over the optimization iterations. It can be integrated with existing SSL methods like FixMatch~\citep{sohn2020fixmatch}. From the view of optimization, we show that eventually the proposed \Alg~can non-asymptotically converge with theoretical guarantee. Empirical evaluations on image benchmarks validate the effectiveness of \Alg~comparing with the state-of-the-art SSL algorithms. 
 
\section{Related Work}
There has been growing interest in semi-supervised learning for training machine learning and deep learning~\citep{flach2012machine,goodfellow2016deep}. A number of SSL methods have been studied by leveraging the structure of unlabeled data including consistency regularization~\citep{bachman2014learning,sajjadi2016regularization,laine2017temporal,tarvainen2017mean,miyato2018virtual,xie2020unsupervised}, entropy minimization~\citep{grandvalet2005semi,lee2013pseudo}, and other interesting approaches~\citep{berthelot2019mixmatch,berthelot2019remixmatch}. In addition, several studies on SSL have been proposed to keep SSL performing safe when using unlabeled data, which is known as safe SSL~\citep{li2014towards}. An non-exhaustive list of those studies include~\citep{cozman2003semi,singh2009unlabeled,li2011improving,balsubramani2015optimally, loog2015contrastive,li2017learning,krijthe2017projected,li2019towards,mey2019improvability,guosafe2020}. For example, \cite{ren2020not} proposed a new SSL framework that uses an individual weight for each unlabeled example, and it updates the individual weights and models iteratively by solving a bi-level optimization problem approximately. In this paper, we mainly focus on improved deep SSL methods with the use of unlabeled data selection. Comprehensive surveys on SSL methods could be refer to~\citep{zhu2005semi,chapelle2006semi,zhu2009introduction,hady2013semi,van2020survey}.

The use of unlabeled data selection by a threshold is not new in the literature of SSL. As a simple yet widely used heuristic algorithm, pseudo-labeling~\citep{lee2013pseudo} (a.k.a. self-training~\citep{mclachlan1975iterative,yarowsky1995unsupervised,rosenberg2005semi,sajjadi2016regularization,laine2017temporal,xie2020self}) uses the prediction model itself to generate pseudo labels for unlabeled images. Then the unlabeled images whose corresponding pseudo label's highest class probability is larger than a predefined threshold will be used for the training. Nowadays, pseudo-labeling has been become an important component of many modern SSL methods~\citep{xie2020unsupervised,sohn2020fixmatch}. 

With the use of weak and strong data augmentations, several recent works such as UDA~\citep{xie2020unsupervised}, ReMixMatch~\citep{berthelot2019remixmatch} and FixMatch~\citep{sohn2020fixmatch} have been proposed in image classification problems. Generally, they use a weakly-augmented~\footnote{Both UDA and ReMixMatch use crop and flip as ``weak" augmentation while FixMatch uses flip and shift.} unlabeled image to generate a pseudo label and enforce consistency against strongly-augmented~\footnote{For ``strong" augmentation, UDA uses RandAugment~\citep{cubuk2020randaugment}, ReMixMatch uses CTAugment~\citep{cubuk2019autoaugment}, and FixMatch uses both.} version of the same image.  In particular, UDA and FixMatch use a fixed threshold to retain the unlabeled example whose highest probability in the predicted class distribution for the pseudo label is higher than the threshold. For example, UDA sets this threshold to be 0.8 for CIFAR-10 and SVHN, and FixMatch sets the threshold to be 0.95 for all data-sets. To encourage the model to generate high-confidence predictions, UDA and ReMixMatch sharpen the guessed label distribution by adjusting its temperature and then re-normalize the distribution. \cite{sohn2020fixmatch} have shown that the sharpening and thresholding pseudo-labeling have a similar effect.

By contrast, the proposed \Alg~method selects a subset of unlabeled data to be used in training models by a data-dependent dynamic threshold, and its theoretical convergence guarantee is established for stochastic gradient descent under the non-convex setting, which is applicable to deep learning.  

\section{Preliminary and Background}
\subsection{Problem Setting}
We study the task of learning a model to map an input $\x \in \X\subseteq\R^d$ onto a label $\y \in \Y$. In many machine learning applications, $\x$ refers to the feature and $\y\in\Y$ refers to the label for classification or regression. For simplicity, let $\xi$ denote the input-label pair $(\x,\y)$, i.e. $\xi := (\x,\y)$. We denote by $\P$ the underlying distribution of data pair $\xi$, then $\xi \sim\P$. The goal is to learn a model $\w\in\R^d$ via minimizing an optimization problem whose objective function $F(\w)$ is the expectation of random loss function $f(\w;\xi)$:
\begin{align}\label{prob:eq:labeled}
    \min_{\w\in\R^d}F(\w) := \E_{\xi\sim \P}\left[f\left(\w; \xi \right)\right],
\end{align}
where $\E_\xi[\cdot]$ is an expectation taking over random variable $\xi\sim\P$. The optimization problem (\ref{prob:eq:labeled}) covers most machine learning and deep learning applications. 
In this paper, we consider the classification problem with $K$-classes, whose loss function is the cross-entropy loss given by
\begin{align}\label{app:loss:CE}
   f(\w;\xi_i) = H(\y_i, \p(\w;\x_i))
   :=\sum_{k=1}^{K} -y_{i,k}\log\left(\frac{\exp(p_k(\w;\x_i))}{\sum_{j=1}^{K}\exp(p_j(\w;\x_i))}\right),
\end{align}
where $\p(\w;\x)$ is the prediction function and $H(\q,\p)$ is the cross-entropy between $\q$ and $\p$. 
In this paper, we do not require the function $f(\w;\xi)$ to be convex in terms of $\w$, which is applicable to various deep learning tasks.   

In SSL, it consists of labeled examples and unlabeled examples. Let
\begin{align}
\D_l := \left\{(\x_i, \y_i), i=1, \ldots, N_l\right\}    
\end{align}
be the {labeled} training data. Given {unlabeled} training examples $\{\x_i^u, i = 1, 2, \dots, N_u\}$, one can generate pseudo label $\yh_i^u$ based on the predictions of a supervised model on labeled data. Different SSL methods such as pseudo-labeling~\citep{lee2013pseudo},  Adversarial Training~\citep{miyato2018virtual}, UDA~\citep{xie2020unsupervised},  
and FixMatch~\citep{sohn2020fixmatch} have been proposed to generate pseudo labels. We denote by 
\begin{align}
\D_u := \left\{(\x_i^u, \yh_i^u), i=1, \ldots, N_u\right\}   
\end{align}
the unlabeled data, where $\yh_i^u \in \Y$ is the pseudo label. Although it contains pseudo label, we still call $\D_u$ unlabeled data for simplicity in our analysis. Usually, the number of unlabeled examples is much larger than the number of labeled examples, i.e., $N_u \gg N_l$. Finally, the training data consists of labeled data $\D_l$ and unlabeled data with pseudo label $\D_u$, and thus the training loss of an SSL algorithm usually contains supervised loss $F_s$ and unsupervised loss $F_u$ with a weight $\lambda_u>0$: $F_s + \lambda_uF_u$, where $F_s$ is constructed on $\D_l$ and $F_u$ is constructed on $\D_u$. In image classification problems, $F_s$ is just the standard cross-entropy loss:
\begin{align}\label{loss:supervised}
    F_s(\w) := \frac{1}{N_l}\sum_{i=1}^{N_l} f(\w;\xi_i),
\end{align}
where $\xi_i \in\D_l$ and $f$ is defined in (\ref{app:loss:CE}). Thus, different constructions of the unsupervised loss $F_u$ lead to different SSL methods. Typically, there are two ways of constructing $F_u$: one is to use pseudo labels to formulate a  ``supervised" loss such as cross-entropy loss (e.g., FixMatch), and another one is to optimize a regularization that does not depend on labels such as consistency regularization (e.g., $\Pi$-Model). Next, we will introduce a recent SSL work to interpret how to generate pseudo labels and construct unsupervised loss $F_u$.

\subsection{FixMatch: An SSL Algorithm with Fixed Thresholding} 
Due to its simplicity yet empirical success, we select FixMatch~\citep{sohn2020fixmatch} as an SSL example in this subsection. Besides, we consider FixMatch as a warm-up of the proposed algorithm, since FixMatch uses a fixed threshold to ratain unlabeled examples and it will be used as a pipeline in the proposed algorithm.

The key idea of FixMatch is to use a separate weak and strong augmentation when generating model's predicted class distribution and one-hot label in unsupervised loss. Specifically, based on a supervised model $\w$ and a weak augmentation $\alpha$, FixMatch predict the class distribution
\begin{align}\label{fixmatch:h}
    \h_i = \p(\w, \alpha(\x_i^{u}))
\end{align}
for a weakly-augmented version of a unlabeled image $\x_i^u$, where $\p(\w, \x)$ is the prediction function. Then it creates a pseudo label by
\begin{align}\label{fixmatch:y}
    \widehat \y_i^u = \arg\max(\h_i).
\end{align}
Following by~\citep{sohn2020fixmatch}, the $\arg\max$ applied to a probability distribution produces a ``one-hot" probability distribution. 
To construct the unsupervised loss, it computes the model prediction for a strong augmentation $\mathcal T$ of the same unlabeled image $\x_i^u$:
\begin{align}\label{fixmatch:p}
     \p(\w,\mathcal T(\x_i^{u})).
\end{align}
The unsupervised loss is defined as the cross-entropy between $\widehat \y_i^u$ and $\p_i $: 
\begin{align}\label{fixmatch:CE}
    H(\widehat\y_i^u,\p(\w,\mathcal T(\x_i^{u})) ).
\end{align}
 Eventually, FixMatch only uses the unlabeled examples with a high-confidence prediction by selecting based on a {\bf fixed threshold} $\tau = 0.95$. Therefore, in FixMatch the cross-entropy loss with pseudo-label and confidence for unlabeled data is given by
\begin{align}\label{fixmatch:loss:unsup}
    F_u(\w) =  \frac{1}{N_u}\sum_{i=1}^{N_u} I(\max(\h_i) \ge \tau) H(\widehat\y_i^u,\p(\w,\mathcal T(\x_i^{u})) ) ,
\end{align}
where $I(\cdot)$ is an indicator function. 

As we discussed in introduction, this fixed threshold may lead to eliminate/select too many unlabeled examples with correct/wrong pseudo labels (see Figure~\ref{fig:dash}), which eventually could drop off overall performance. It is a natural choice: the threshold is not fixed across the optimization iterations. Thus, in the next section, we are going to propose a new SSL scheme having a dynamic threshold.

\begin{algorithm*}[t]
\caption{Dash: Semi-Supervised Learning with {\bf D}yn{\bf a}mic Thre{\bf sh}olding}\label{alg:dash}
\begin{algorithmic}
\STATE Input: learning rate $\eta_0$ and mini-batch size $m_0$ for stage one, learning rate $\eta$ and parameter $m$ of mini-batch size for stage two, two parameters $C > 1$ and $\gamma > 1$ for computing threshold, and violation probability $\delta$.
\STATE {\color{gray}// Warm-up Stage: run SGD in $T_0$ iterations.}
\STATE Initialization: $\u_0 = \w_0$
\FOR{$t=0, 1, \ldots, T_0-1$}
    \STATE Sample $m_0$ examples $\xi_{t,i}$ $(i=1,\dots,m_0)$ from $\D_l$,
    \STATE $\u_{t+1} = \u_{t} - \eta_0 \tilde\g_{t}$ where $\tilde\g_{t} = \frac{1}{m_0}\sum_{i=1}^{m_0}\nabla f_s(\u_{t};\xi_{t,i})$
\ENDFOR
\STATE {\color{gray}// Selection Stage: run SGD in $T$ iterations.}
\STATE Initialization:  $\w_{1} = \u_{T_0}$.
\STATE Compute the value of $\widehat{\rho}$ as in (\ref{eqn:rho:hat}). {\color{gray}// In practice, $\widehat{\rho}$ can be obtained as in (\ref{eqn:rho:0}).} 
\FOR{$t=1, \ldots, T$}
	\STATE 1) Sample $n_t = m\gamma^{t-1}$ examples from $\D_u$, where the pseudo labels in $\D_u$ are generated by FixMatch 
    \STATE 2) Set the threshold $\rho_t = C\gamma^{-(t-1)}\widehat{\rho}$.
    \STATE 3) Compute truncated stochastic gradient $\g_t$ as (\ref{grad:truncated}). 
    \STATE 4) Update solution by SGD using stochastic gradient $\g_t$ and learning rate $\eta$: $\w_{t+1} = \w_t - \eta \g_t$.
\ENDFOR
\STATE Output: $\w_{T+1}$
\end{algorithmic}
\end{algorithm*}

\section{Dash: An SSL Algorithm with Dynamic Thresholding}
Before introducing the proposed method, we would like to point out the importance of unlabeled data selection in SSL from the theoretical view of optimization. Classical SSL methods~\citep{zhu2005semi,chapelle2006semi,zhu2009introduction,hady2013semi,van2020survey} assume that labeled data and unlabeled data are from the same distribution. That is to say, $\xi\sim\P$ holds for $\xi\in\D_l\cup\D_u$. Then SSL methods aim to solve the optimization problem (\ref{prob:eq:labeled}) by using a standard stochastic optimization algorithm like mini-batch stochastic gradient descent (SGD). Specifically, at iteration $t$, mini-batch SGD updates intermediate solutions by
\begin{align}
    \w_{t+1} = \w_{t} - \frac{\eta}{m}\sum_{i=1}^{m}\nabla f(\w_{t};\xi_{t,i}),
\end{align}
where $m$ is the mini-batch size, $\xi_{t,i}$ is sampled from training data $\D_l\cup\D_u$, $\nabla f(\w;\xi)$ is the gradient of $f(\w;\xi)$ in terms of $\w$. In this situation, the theoretical convergence guarantee of SSL algorithms can be simply established under mild assumptions on objective function $f(\w;\xi)$ such as smoothness and bounded variance~\citep{ghadimi2016mini}. If the labeled data and unlabeled data are not from the same distribution such as some of pseudo labels are not correct, classical SSL methods with standard stochastic optimization algorithm may lead to the performance drops~\citep{chapelle2006semi,oliver2018realistic}. Besides, the theoretical guarantee of the optimization algorithm for this case is not clear. This inspires us to design a new algorithm to overcome this issue. To this end, we proposed an SSL method that can dynamically select unlabeled examples during the training progress.

First, let us define the loss function for the proposed method. Same as FixMatch, the supervised loss $F_s(\w)$ for the proposed method is the standard cross-entropy loss on labeled data~$\D_l$: 
\begin{align}\label{dash:sup}
    F_s(\w) := \frac{1}{N_l}\sum_{i=1}^{N_l} f_s(\w;\xi_i),
\end{align}
where $\xi_{i} = (\x_{i},\y_{i})$ is sampled from $\D_l$, $f_s(\w;\xi_i) = H(\y_i, \p(\w;\alpha(\x_i)))$, and $\alpha(\x)$ is the weakly-augmented version of $x$. 
Since the new dynamic threshold is not fixed, we let it rely on the optimization iteration $t$ and it is denoted by $\rho_t$. Then the unsupervised loss is given by
\begin{align}\label{dash:unsup}
    F_u(\w) = \frac{1}{N_u}\sum_{i=1}^{N_u} I(f_u(\w; \xi_{i}^u) \le \rho_t) f_u(\w; \xi_{i}^u),
\end{align}
where $\xi_{i}^u = (\x^u_{i}, \widehat \y^u_{i})$ is sampled from $\D_u$,  $f_u(\w; \xi_{i}^u) = H(\widehat \y_{i}^u, \p(\w; \mathcal T(\x_{i}^u)) )$, $\mathcal T(x)$ is the strongly-augmented version of $x$, and the pseudo label $\widehat\y_{i}^u$ is generated based on (\ref{fixmatch:h}) and (\ref{fixmatch:y}) using prediction model $\w$ and a unlabeled image $\x_{i}^u$. The unsupervised loss (\ref{dash:unsup}) shows that the \Alg~will retain the unlabeled example whose loss is smaller than the threshold $\rho_t$.
If we rewrite the indicator function $I(\max(\h_i) \ge \tau)$ in (\ref{fixmatch:loss:unsup}) to an equivalent expression
\begin{align}
    I(-\log(\max(\h_i)) \le -\log(\tau)),
\end{align}
we can consider $-\log(\max(\h_i))$ as a cross-entropy loss for one-hot label. Roughly speaking, FixMatch retains the unlabeled images with the loss $-\log(\max(\h_i))$ smaller than $-log(0.95) \approx 0.0513$. It is worth nothing that the loss $-\log(\max(\h_i))$ contains the information of weakly-augmented images, while the loss $f_u(\w; \xi_{i}^u)$ in (\ref{dash:unsup}) includes the information of both  weakly-augmented and strongly-augmented images, meaning that the proposed method considers the entire loss function.

\begin{figure}[t]
    \centering
    \includegraphics[width=0.5\textwidth]{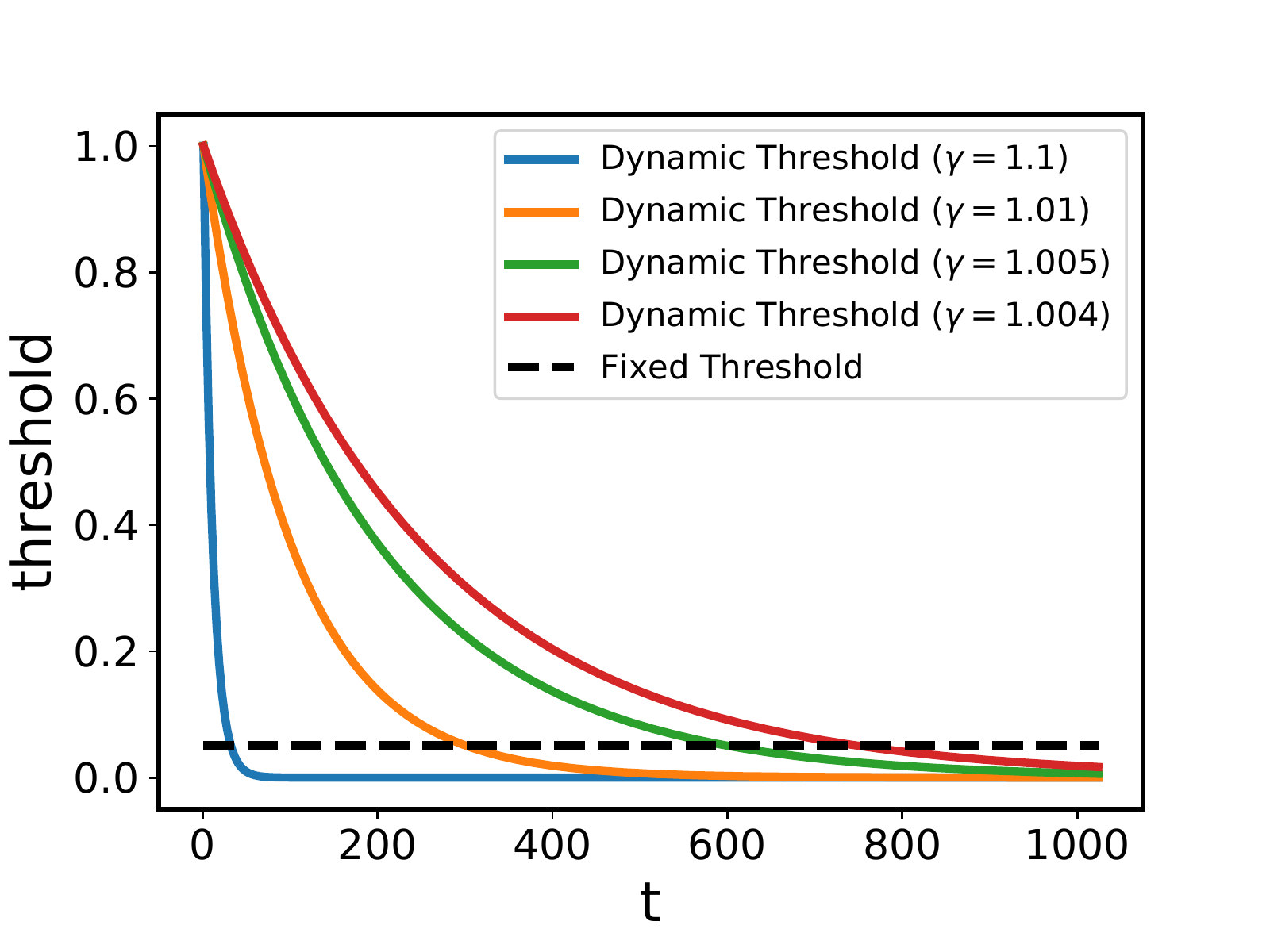}
    \caption{Comparison of fixed threshold and dynamic threshold. Fixed threshold is in the scale of negative log: $-\log(0.95)$, dynamic threshold $\rho_t = 1.0001 \gamma^{-(t-1)}$.}\label{fig:threshold}
\end{figure}
We then need to construct $\rho_t$. Intuitively, with the increase of the optimization iteration $t$, the loss function would decrease in general, so that $\rho_t$ is also required to decrease. Mathematically, we set the {\bf dynamic threshold} $\rho_t$ as a decreasing function of $t$, which is given by
\begin{align}\label{rho:t}
\rho_t := C\gamma^{-(t-1)}\widehat{\rho},    
\end{align}
where $C>1, \gamma>1$ are two constants. For example, we set $C=1.0001$ in our experiments and thus at first iteration (i.e., $t=1$) the unlabeled examples whose loss values are smaller than $\rho_t = 1.0001 \times\widehat\rho$ will be used in the training. Figure~\ref{fig:threshold} shows a comparison of fixed threshold used in FixMatch and dynamic threshold $\rho_t$ in (\ref{rho:t}) with $C=1.0001, \widehat\rho=1$ and different $\gamma$, where the threshold in FixMatch is in the scale of negative log. It seems our thresholding strategy matches the curve of training loss in many real applications (e.g., see Figure 1 (a) of~\citep{zhang2016understanding}). 

Next, it is important to estimate the value of $\widehat\rho$. In theory, it can be estimated by \begin{align}\label{eqn:rho:hat}
\widehat{\rho} = \max \left\{ a, \frac{4G^2\left(1+\delta b_0 m\right) }{\delta \mu a_0m} \right\},    
\end{align}
where contains several parameters related to the property of considered problem~(\ref{prob:eq:labeled}) whose detailed definitions can be found in Theorem~\ref{thm:main}. Please note that the estimation of $\widehat\rho$ in (\ref{eqn:rho:hat}) is for the use of convergence analysis only. In practice, we can use the following averaged loss from the training set $\D_l$ as the approximate $\widehat\rho$:
\begin{align}\label{eqn:rho:0}
    \widehat\rho \approx \frac{1}{|\D_l|}\sum_{ \xi_i \in \D_l} f(\w_1; \xi_i),
\end{align} 
where $|\D_l|$ is the number of examples in $\D_l$, and $\w_1$ can be learned on $\D_l$. We can see from (\ref{eqn:rho:0}) with (\ref{dash:unsup}) and (\ref{rho:t}) that the unlabeled examples whose losses are smaller than the averaged loss of labeled examples will be maintained during the training process.

Finally, it is ready to describe the proposed \Alg~algorithm in details that contains two stages: warm-up stage and selection stage. In the warm-up stage, it runs SGD to train a model over labeled data $\D_l$ in certain steps.  Not only for warm-up, this stage is also used for estimating $\widehat\rho$ in (\ref{eqn:rho:0}). 
In the selection stage, we conduct SGD against $\D_u$ using $\w_1$ as the initial solution. At each iteration $t$, we sample $n_t = m\gamma^{t-1}$ training examples from $\D_u$, where $m>1$ is a parameter defined in (\ref{def:m}). We compute the stochastic gradients according to (\ref{dash:unsup}):
\begin{align}\label{grad:truncated}
    \g_t = \frac{\sum_{i=1}^{n_t} I(f_u(\w_t; \xi_{t,i}^u) \leq \rho_t) \nabla f_u(\w_t; \xi_{t,i}^u) }{\sum_{i=1}^{n_t} I(f_u(\w_t; \xi_{t,i}^u) \leq \rho_t)}.
\end{align}
Since $N_l$ is small, in practice we can also construct the stochastic gradient by using all labeled data as 
\begin{align}\label{grad:truncated:3}
    \g_t = \frac{\sum_{i=1}^{n_t-N_l} I(f_u(\w_t; \xi_{t,i}^u) \leq \rho_t) \nabla f_u(\w_t; \xi_{t,i}^u) }{N_l+\sum_{i=1}^{n_t-N_l} I(f_u(\w_t; \xi_{t,i}^u) \leq \rho_t)}+\frac{\sum_{i=1}^{N_l} \nabla f_s(\w_t; \xi_{t,i})}{N_l+\sum_{i=1}^{n_t-N_l} I(f_u(\w_t; \xi_{t,i}^u) \leq \rho_t) },
\end{align}
where $\xi^u_{t,i} = (x^u_{t,i}, y^u_{t,i}) \in \D_u$ and $\xi_{t,i} = (x_{t,i}, y_{t,i}) \in \D_l$. 
The solution is then updated by mini-batch SGD, whose update step is given by
\begin{align}
\w_{t+1} = \w_t - \eta \g_t.
\end{align}
The detailed updating steps of the proposed algorithm are presented in Algorithm~\ref{alg:dash}, which called SSL with {\bf D}yn{\bf a}mic Thre{\bf sh}olding (\Alg).

\section{Convergence Result}
To establish the convergence result of the proposed Dash algorithm, we need to give some preliminaries. 
Recall that the training examples for the labeled data $\D_l$ follow the distribution $\P$, and we aim to minimize the optimization problem (\ref{prob:eq:labeled}).
For the examples coming from the unlabeled data $\D_u$, suppose it is a mixture of two distributions, $\P$ and $\Q$. More specifically, with a probability $q$, we will sample an example from $\P$ and with a probability $1 - q$ sample from $\Q$:
\begin{align}\label{dist:unlabeled}
    \xi\sim q \P + (1-q)\Q,~\text{where}~\xi\in\D_u,~q\in(0,1).
\end{align}
We define the objective function $B(\w)$ as the expected loss for distribution $\Q$, i.e.
\begin{align}\label{prob:eq:unlabeled}
    B(\w) := \E_{\xi\sim \Q}\left[f\left(\w; \xi \right)\right].
\end{align}

For the simplicity of convergence analysis, we do not consider the weak and strong augmentations, i.e., let $f_s=f_u=f$ in (\ref{dash:sup}) and (\ref{dash:unsup}). Without loss of generality, we assume that our loss function is non-negative and is bounded by $1$, i.e. $f(\w; \xi) \in [0, 1]$ for any $\w$ and $\xi$. Then by (\ref{prob:eq:labeled}) and (\ref{prob:eq:unlabeled}), we have $F(\w) \in [0, 1]$ and $B(\w) \in [0, 1]$. In order to differentiate the two distribution, we follow the idea of Tsybakov noisy condition~\citep{mammen1999smooth,tsybakov2004optimal}, and assume, for any solution $\w$, if $F(\w) \leq a$, then
\begin{align}\label{eqn:condition}
    \E_{\xi\sim\Q}\left[I_{\left\{\xi:f(\w;\xi) \leq F(\w)\right\}}(\xi)\right] \leq bA^{\theta}(\w), 
\end{align}
where $I_S(\xi)$ is an indicator function, $\theta \geq 1$, and $b$ is constant.

Finally, we made a few more assumptions that are commonly used in the studies of non-convex optimization (e.g., deep learning)~\citep{ghadimi2013stochastic,yuan2019stagewise}. Throughout this paper, we also make the assumptions on the problem (\ref{prob:eq:labeled}) as follows.
\begin{ass}\label{ass:2}
Assume the following conditions hold: 
\begin{itemize} 
\item[(i)]  The stochastic gradient $\nabla f(\w;\xi)$ is unbiased, i.e., $$\E_{\xi\sim\P}[\nabla f(\w;\xi)] = \nabla F(\w),$$ and there exists a constant $G>0$, such that $$\| \nabla f(\w;\xi)\|\leq G.$$ 
\item[(ii)] $F(\w)$ is smooth with a $L$-Lipchitz continuous gradient, i.e., it is differentiable and there exists a constant $L>0$ such that $$\|\nabla F(\w)  - \nabla F(\u)\|\leq L\|\w - \u\| ,\forall \w, \u \in\R^d.$$
\end{itemize}
\end{ass}
Assumption~\ref{ass:2} (i) assures that the stochastic gradient of the objective function is unbiased and the gradient of $f(\w; \xi)$ in terms of $\w$ is upper bounded. Assumption~\ref{ass:2} (ii) says the objective function is $L$-smooth, and it has an equivalent expression which is $\forall \w, \u \in \R^d$, 
\begin{align*}
 F(\w) - F(\u) \le \langle \nabla F(\u), \w - \u \rangle + \frac{L}{2}\|\w-\u\|^2.
\end{align*}   

We now introduce an important property regarding $F(\w)$, i.e. the Polyak-{\L}ojasiewicz (PL) condition~\citep{polyak1963gradient} of $F(\w)$. 
\begin{ass}\label{ass:3}
There exists $\mu>0$ such that $$2\mu (F(\w) - F(\w_*)) \le \|\nabla F(\w)\|^2, \forall \w\in\R^d.$$
\end{ass}
This PL property has been theoretically and empirically observed in training deep neural networks~\citep{allen2019convergence, yuan2019stagewise}. This condition is widely used to establish convergence in the literature of non-convex optimization, please see~\citep{yuan2019stagewise, wang2019spiderboost, karimi2016linear, li2018simple, charles2018stability} and references therein.  

Now, we are ready to provide the theoretical result for \Alg. Without loss of generality, let $F(\w_*) = 0$ in the analysis. Please note that this is a common property observed in training deep neural networks~\citep{zhang2016understanding,allen2019convergence,du2019gradient,arora2019fine,chizat2019lazy, hastie2019surprises,yun2019small}.
The following theorem states the convergence guarantee of the proposed \Alg~algorithm. We include its proof in the Appendix. 
\begin{thm}\label{thm:main}
Under Assumptions~\ref{ass:2} and \ref{ass:3}, suppose that $C>1$ and $F(\w_*) = 0$, for any $\delta \in (0,1)$, $\eta_0 L\leq 1$, $\eta L\leq 1$, let $T_0 = \frac{\log(2F(\w_0)/a)}{\log(1/(1-\eta_0\mu))}$, $m_0 =  \frac{4 G^2}{\delta \mu a}$,
\begin{align}
m = &\left\lceil \max\left(\sqrt{\frac{\log(2/\delta)}{q^2}}, \sqrt{\frac{\log(2/\delta)}{(1-q)^2}}, \sqrt{\frac{\log(2/\delta)}{q(1-C^{-1})^2}}\right) \right\rceil, \label{def:m}\\
\widehat{\rho} =& \max \left\{ a, \frac{4G^2\left(1+\delta b_0 m\right) }{\delta \mu a_0m} \right\} \label{def:rho:hat}
\end{align}
in Algorithm~\ref{alg:dash}, then with a probability $1 - (4T+1)\delta$, we have
    $F(\w_{T+1}) \leq \widehat{\rho}\gamma^{-T}$.
\end{thm}
{\bf Remark. } We can see from the above result that one can set  the iteration number $T$ to be large enough to ensure the convergence of \Alg. Specifically, in order to have an $\epsilon$ optimization error, one can set $T = \log(\widehat\rho/\epsilon)/\log(\gamma)$, then $F(\w_{T+1}) \leq \epsilon$. The total sample complexity of \Alg~is 
$T_0m_0 + \sum_{t=1}^{T}m\gamma^{t-1} \le T_0m_0 + \frac{m\gamma^T}{\gamma - 1} 
= T_0m_0 + \frac{m \widehat\rho}{\epsilon(\gamma-1)} = O(1/\epsilon).$ This rate matches the result of supervised learning in~\citep{karimi2016linear} when analyzing the standard SGD under Assumptions~\ref{ass:2} and \ref{ass:3}.

\begin{table*}[t]
\caption{Comparison of top-1 testing error rates for different methods using Wide ResNet-28-2 for CIFAR-10, Wide ResNet-28-8 for CIFAR-100 (in $\%$, mean $\pm$ standard deviation).}\label{table:CIFAR}
\begin{center}
  \def\sym#1{\ifmmode^{#1}\else\(^{#1}\)\fi}
  \begin{tabular}{cccccccccc}
    \hline
   &  \multicolumn{3}{c}{CIFAR-10} & \multicolumn{3}{c}{CIFAR-100}  \\
     \cmidrule(lr){2-4}\cmidrule(lr){5-7} 
  Algorithm  & \multicolumn{1}{c}{40 labels} & \multicolumn{1}{c}{250 labels} & \multicolumn{1}{c}{4000 labels} & \multicolumn{1}{c}{400 labels}
     & \multicolumn{1}{c}{2500 labels} & \multicolumn{1}{c}{10000 labels}  \\
    \hline
    $\Pi$-Model & - & 54.26$\pm$3.97 & 14.01$\pm$0.38 & - & 57.25$\pm$0.48 & 37.88$\pm$0.11  \\
    Pseudo-Labeling & - & 49.78$\pm$0.43 & 16.09$\pm$0.28 & - & 57.38$\pm$0.46 & 36.21$\pm$0.19 \\
    Mean Teacher & - & 32.32$\pm$2.30 & 9.19$\pm$0.19 & - & 53.91$\pm$0.57 & 35.83$\pm$0.24 \\
    MixMatch & 47.54$\pm$11.50 & 11.05$\pm$0.86 & 6.42$\pm$0.10 & 67.61$\pm$1.32 & 39.94$\pm$0.37 & 28.31$\pm$0.33 \\
    UDA & 29.05$\pm$5.93 & 8.82$\pm$1.08 & 4.88$\pm$0.18 & 59.28$\pm$0.88 &  33.13$\pm$0.22 & 24.50$\pm$0.25  \\
    ReMixMatch & 19.10$\pm$9.64 & 5.44$\pm$0.05 & 4.72$\pm$0.13 & {\bf44.28}$\pm$2.06 & {27.43}$\pm$0.31 & 23.03$\pm$0.56  \\
    RYS (UDA) &- & 5.53$\pm$0.17 & 4.75$\pm$0.28 & - & - & -  \\
    RYS (FixMatch)& - & 5.05$\pm$0.12 & 4.35$\pm$0.06 & - &- & -  \\
    \hline
    FixMatch (CTA) & 11.39$\pm$3.35 & 5.07$\pm$0.33 & 4.31$\pm$0.15 & 49.95$\pm$3.01 & 28.64$\pm$0.24 & 23.18$\pm$0.11 \\  
  {\bf Dash (CTA, ours)}& {\bf9.16}$\pm$4.31  &   4.78$\pm$0.12 &   4.13$\pm$0.06 &   44.83$\pm$1.36   & 27.85$\pm$0.19    &   22.77$\pm$0.21   \\
    \hline
    FixMatch (RA) & 13.81$\pm$3.37 & 5.07$\pm$0.65& 4.26$\pm$0.05& 48.85$\pm$1.75& 28.29$\pm$0.11&22.60$\pm$0.12\\
    {\bf Dash (RA, ours)}& 13.22$\pm$3.75& {\bf4.56}$\pm$0.13& {\bf4.08}$\pm$0.06& 44.76$\pm$0.96 & {\bf27.18}$\pm$0.21& {\bf 21.97}$\pm$0.14\\
    \hline
  \end{tabular}
  \end{center}
\end{table*}

\section{Experiments}
In this section, we present some experimental results for image classification tasks. To evaluate the efficacy of \Alg, we compare it with several state-of-the-art (SOTA) baselines on several standard SSL image classification benchmarks including CIFAR-10, CIFAR-100~\citep{krizhevsky2009learning}, SVHN~\citep{Netzer201137648}, and STL-10~\citep{coates2011analysis}. Specifically, SOTA baselines are MixMatch~\citep{berthelot2019mixmatch}, UDA~\citep{xie2020unsupervised}, ReMixMatch~\citep{berthelot2019remixmatch}, FixMatch~\citep{sohn2020fixmatch} and the algorithm RYS from~\citep{ren2020not}\footnote{Since the authors did not name their algorithm, we use RYS to denote their algorithm for simplicity, where RYS is the combination of initials for 
last names of the authors.}. Besides, $\Pi$-Model~\citep{rasmus2015semi}, Pseudo-Labeling~\citep{lee2013pseudo} and Mean Teacher~\citep{tarvainen2017mean} are included in the comparison. 

\subsection{Data-sets} 
The original CIFAR data-sets have 50,000 training images and 10,000 testing images of 32$\times$32 resolutions, and CIFAR-10 has 10 classes containing 6,000 images each, while CIFAR-100 has 100 classes containing 600 images each. The original SVHN data-set has 73,257 digits for training and 26,032 digits for testing, and the total number of classes is 10. The original STL-10 data set has 5,000 labeled images from 10 classes and 100,000 unlabeled images, which contains out-of-distribution unlabeled images. 

Following by~\citep{sohn2020fixmatch}, we train ten benchmarks with different settings: CIFAR-10 with 4, 25, or 400 labels per class, CIFAR-100 with 4, 25, or 100 labels per class, SVHN with 4, 25, or 100 labels per class, and the STL-10 data set. For example, the benchmark CIFAR-10 with 4 labels per class means that there are 40 labeled images in CIFAR-10 and the remaining images are unlabeled, and then we denote this data set by CIFAR-10 with 40 labels. For fair comparison, same sets of labeled images from CIFAR, SVHN and STL-10 were used for the proposed \Alg~ method and other baselines in all experiments.

\subsection{ Models and Hyper-parameters}
We use the Wide ResNet-28-2 model~\citep{zagoruyko2016wide} as the backbone for CIFAR-10 and SVHN, Wide ResNet-28-8 for CIFAR-100, and Wide ResNet-37-2 for STL-10. In the proposed \Alg, we use FixMatch~\footnote{In our experiments, the FixMatch codebase is used: \url{https://github.com/google-research/fixmatch}} as our pipeline to generate pseudo labels and to construct supervised and unsupervised losses. 
We employ CTAugment (CTA)~\citep{cubuk2019autoaugment} and RandAugment (RA)~\citep{cubuk2020randaugment} for the strong augmentation scheme following by~\citep{sohn2020fixmatch}. Similar to~\citep{sohn2020fixmatch}, we use the same training protocol such as optimizer, learning rate schedule, data preprocessing, random seeds, and so on.

The total number of training epochs is set to be 1024 and the mini-bach size is fixed as 64. For the value of weight decay, we use $5\times 10^{-4}$ for CIFAR-10, SVHN and STL-10, $1\times10^{-3}$ for CIAR-100. The SGD with momentum parameter of $0.9$ is employed as the optimizer. The cosine learning rate decay schedule~\citep{loshchilov2016sgdr} is used as~\citep{sohn2020fixmatch}. The initial learning rate is set to be $0.06$ for all data-sets. We use (\ref{grad:truncated}) to compute stochastic gradients. 

At the first $10$ epochs, we do not implement the selection scheme and thus the algorithm uses all selected training examples, meaning that the threshold is infinite, i.e., $\rho_t = \infty$. After that, we use the threshold to select unlabeled examples, and we choose $\gamma = 1.27$ in $\rho_t$ to reduce the dynamic threshold until its value to be $0.05$. That is to say, in practice we give a minimal value of dynamic threshold, which is $0.05$\footnote{In practice, we use $\rho_t = \max\{\rho_t, 0.05\}$.}. We fix the constant $C$ as $1.0001$ and estimate the value of $\widehat\rho$ by using (\ref{eqn:rho:0}). We decay the dynamic threshold every 9 epochs. We use the predicted label distribution as soft label during the training and it is sharpened by adjusting its temperature of $0.5$, which is similar to MixMatch. Once the dynamic threshold is reduced to $0.05$, we turn it to one-hot label in the training since the largest label probability is close to $1$.

\begin{table*}[t]
\caption{Comparison of top-1 testing error rates for different methods using Wide ResNet-28-2 for SVHN and Wide ResNet-37-2 for STL-10 (in $\%$, mean $\pm$ standard deviation).}\label{table:SVHN}
\begin{center}
  \def\sym#1{\ifmmode^{#1}\else\(^{#1}\)\fi}
  \begin{tabular}{ccccccccccc}
    \hline
    & \multicolumn{3}{c}{SVHN} & \multicolumn{1}{c}{STL-10} \\
     \cmidrule(lr){2-4}  \cmidrule(lr){5-5}
  Algorithm   & \multicolumn{1}{c}{40 labels} & \multicolumn{1}{c}{250 labels} & \multicolumn{1}{c}{1000 labels} & \multicolumn{1}{c}{1000 labels} \\
    \hline
    $\Pi$-Model &  - & 18.96$\pm$1.92 & 7.54$\pm$0.36  & 26.23$\pm$0.82 \\
    Pseudo-Labeling &  - & 20.21$\pm$1.09 & 9.94$\pm$0.61 & 27.99$\pm$0.83\\
    Mean Teacher &  - & 3.57$\pm$0.11 & 3.42$\pm$0.07 & 21.43$\pm$2.39\\
    MixMatch &  42.55$\pm$14.53 & 3.98$\pm$0.23 & 3.50$\pm$0.28 & 10.41$\pm$0.61 \\
    UDA &  52.63$\pm$20.51 & 5.69$\pm$2.76& 2.46$\pm$0.24 & 7.66$\pm$0.56\\
    ReMixMatch &  3.34$\pm$0.20 & 2.92$\pm$0.48 & 2.65$\pm$0.08 & 5.23$\pm$0.45 \\
    RYS (UDA) &- & 2.45$\pm$0.08 &  2.32$\pm$0.06 & - \\
    RYS (FixMatch)&  - & 2.63$\pm$0.23 &   2.34$\pm$0.15 & -  \\
   \hline
    FixMatch (CTA) &  7.65$\pm$7.65 & 2.64$\pm$0.64 & 2.36$\pm$0.19 & 5.17$\pm$0.63 \\  
  {\bf Dash (CTA, ours)}&    3.14$\pm$1.60 &   2.38$\pm0.29$ &   2.14$\pm$0.09 & {\bf3.96}$\pm$0.25 \\
    \hline
    FixMatch (RA) &  3.96$\pm$2.17 & 2.48$\pm$0.38 & 2.28$\pm$0.11 & 7.98$\pm$1.50 \\  
  {\bf Dash (RA, ours)}&    {\bf3.03}$\pm$1.59 &   {\bf2.17}$\pm$0.10 &   {\bf2.03}$\pm$0.06  & 7.26$\pm$0.40  \\
  \hline
  \end{tabular}
  \end{center}
\end{table*}

\subsection{ Results} We report the top-1 testing error rates of the proposed \Alg~along within other baselines for CIFAR in Table~\ref{table:CIFAR} and for SVHN and STL-10 in Table~\ref{table:SVHN}, where all the results of baselines are from~\citep{sohn2020fixmatch} except that the results of RYS are from~\citep{ren2020not}. All top-1 testing error rates are averaged over 5 independent random trails with their standard deviations using the same random seeds as baselines used. 

We can see from the results that the proposed \Alg~method has the best performance on CIFAR-10, SVHN and STL-10. For CIFAR-100, the proposed \Alg~is comparable to ReMixMatch, where ReMixMatch performs a bit better on 400 labels and \Alg~using RA is a bit better on 2500 labels and 10000 labels. This reason is that the proposed \Alg~uses FixMatch as its pipeline, and ReMixMatch uses distribution alignment (DA) to encourages the model to predict balanced class distribution (the class distribution of CIFAR-100 is balanced), while FixMatch and \Alg~do not use DA. We further conduct 
\Alg~with DA technique on CIFAR-100 with 400 labels, and the top-1 testing error rate is $43.31\%$, which is better than ReMixMatch ($44.28\%$). We also find that Dash performs well on the data set with out-of-distribution unlabeled images, i.e., STL-10. The result in Table~\ref{table:SVHN} shows that \Alg~with CTA has the SOTA performance of $3.96\%$ on top-1 testing error rate.

Besides, the proposed \Alg~can always outperform FixMatch, showing that the use of dynamic threshold is important to the overall performance. We find the proposed \Alg~ has large improvement when the labeled examples is small (the data-sets with 4 labels per classes), comparing to FixMatch. By using CTA, on CIFAR-100 with 400 labels, on CIFAR-10 with 40 labels, and on SVHN with 40 labels, the proposed \Alg~method outperforms FixMatch result more than $19\%$, $10\%$, and $58\%$ in the terms of top-1 testing error rate, respectively. While by using RA, the corresponding improved rates are $4\%$, $8\%$, and $23\%$ respectively. These results reveal that the dynamic unlabeled example selection is an important term in SSL when the labeled data is small. 

\subsection{Ablation study}
\begin{table}[t]
\centering
\caption{Comparison of top-1 testing error rates for different values of $\gamma$ on CIFAR-10 (in $\%$).}\label{table:diff:gamma}
\begin{tabular}{ccccc}
\hline
$\gamma$ & 1.01 & 1.1 & 1.2 & 1.3 \\ 
\hline
250 labels &  4.85 & {\bf4.76} & 4.99 & 4.82  \\
4000 labels & 4.39 &  4.28 & {\bf 4.11} & 4.31\\
\hline
\end{tabular}
\end{table}
\begin{table}[t]
\centering
\caption{Comparison of top-1 testing error rates for PL and Dash with PL on CIFAR-10 (in $\%$).}\label{table:PL}
\begin{tabular}{ccc}
\hline
Algorithm & PL & Dash-PL \\ 
\hline
250 labels & 49.78 & {\bf 46.90}  \\
4000 labels & 16.09 & {\bf 15.59} \\
\hline
\end{tabular}
\end{table}
In this subsection, we provide two ablation studies using data sets CIFAR-10 with 250 labels and CIFAR-10 with 4000 labels. The first one is to use different $\gamma$ in the dynamic threshold, and the second one is to change FixMatch to Pseudo-Labeling as the pseudo label generator in Dash. 

{\bf Different values of $\gamma$.} Since $\gamma$ is a key component of the dynamic threshold, we conduct an ablation study on different values of $\gamma$ in Dash. For simplicity, we only implement the CTA case. We try four different values of $\gamma \in \{1.01,1.1,1.2, 1.3\}$ and summarize the results in Table~\ref{table:diff:gamma}. Comparing these results with that in Table~\ref{table:CIFAR}, we will find that the choice of $\gamma = 1.27$ in the previous subsection is not the best one. The results also show that Dash is not so sensitive to $\gamma$ in a certain range.

{\bf Dash with Pseudo-Labeling.} Since Dash can be integrated with many existing SSL methods, we use Pseudo-Labeling (PL)~\citep{lee2013pseudo} as the pipeline to generate pseudo labels in Dash. The results are listed in Table~\ref{table:PL}, showing that Dash can improve PL, especially when the labeled images is small.

\section{Conclusion}
We propose a method \Alg~that dynamically selects unlabeled data examples to train learning models. Its selection strategy keeps the unlabeled data whose loss value does not exceed a dynamic threshold at each optimization step. The proposed \Alg~method is a generic scheme that can be easily integrated with existing SSL methods. We demonstrate the use of dynamically selecting unlabeled data can help to the performance of existing SSL method FixMatch in the semi-supervised image classification benchmarks, indicating the importance of dynamic threshold in SSL. The theoretical analysis shows the convergence guarantee of the proposed \Alg~under the non-convex optimization setting.

\section*{Acknowledgements}
The authors would like to thank the anonymous ICML 2021 reviewers for their helpful comments. 

\bibliographystyle{plainnat} 
\bibliography{ref}

\begin{thebibliography}{78}
\providecommand{\natexlab}[1]{#1}
\providecommand{\url}[1]{\texttt{#1}}
\expandafter\ifx\csname urlstyle\endcsname\relax
  \providecommand{\doi}[1]{doi: #1}\else
  \providecommand{\doi}{doi: \begingroup \urlstyle{rm}\Url}\fi

\bibitem[Allen-Zhu et~al.(2019)Allen-Zhu, Li, and Song]{allen2019convergence}
Zeyuan Allen-Zhu, Yuanzhi Li, and Zhao Song.
\newblock A convergence theory for deep learning via over-parameterization.
\newblock In \emph{International Conference on Machine Learning}, pages
  242--252, 2019.

\bibitem[Arora et~al.(2019)Arora, Du, Hu, Li, and Wang]{arora2019fine}
Sanjeev Arora, Simon Du, Wei Hu, Zhiyuan Li, and Ruosong Wang.
\newblock Fine-grained analysis of optimization and generalization for
  overparameterized two-layer neural networks.
\newblock In \emph{International Conference on Machine Learning}, pages
  322--332, 2019.

\bibitem[Bachman et~al.(2014)Bachman, Alsharif, and
  Precup]{bachman2014learning}
Philip Bachman, Ouais Alsharif, and Doina Precup.
\newblock Learning with pseudo-ensembles.
\newblock In \emph{Advances in Neural Information Processing Systems}, pages
  3365--3373, 2014.

\bibitem[Baird(1992)]{baird1992document}
Henry~S Baird.
\newblock Document image defect models.
\newblock In \emph{Structured Document Image Analysis}, pages 546--556.
  Springer, 1992.

\bibitem[Balcan and Blum(2005)]{balcan2005pac}
Maria-Florina Balcan and Avrim Blum.
\newblock A pac-style model for learning from labeled and unlabeled data.
\newblock In \emph{International Conference on Computational Learning Theory},
  pages 111--126. Springer, 2005.

\bibitem[Balsubramani and Freund(2015)]{balsubramani2015optimally}
Akshay Balsubramani and Yoav Freund.
\newblock Optimally combining classifiers using unlabeled data.
\newblock In \emph{Conference on Learning Theory}, pages 211--225, 2015.

\bibitem[Ben-David et~al.(2008)Ben-David, Lu, and P{\'a}l]{ben2008does}
Shai Ben-David, Tyler Lu, and D{\'a}vid P{\'a}l.
\newblock Does unlabeled data provably help? worst-case analysis of the sample
  complexity of semi-supervised learning.
\newblock In \emph{Conference on Learning Theory}, pages 33--44, 2008.

\bibitem[Berthelot et~al.(2019{\natexlab{a}})Berthelot, Carlini, Cubuk,
  Kurakin, Sohn, Zhang, and Raffel]{berthelot2019remixmatch}
David Berthelot, Nicholas Carlini, Ekin~D Cubuk, Alex Kurakin, Kihyuk Sohn, Han
  Zhang, and Colin Raffel.
\newblock Remixmatch: Semi-supervised learning with distribution matching and
  augmentation anchoring.
\newblock In \emph{International Conference on Learning Representations},
  2019{\natexlab{a}}.

\bibitem[Berthelot et~al.(2019{\natexlab{b}})Berthelot, Carlini, Goodfellow,
  Papernot, Oliver, and Raffel]{berthelot2019mixmatch}
David Berthelot, Nicholas Carlini, Ian Goodfellow, Nicolas Papernot, Avital
  Oliver, and Colin~A Raffel.
\newblock Mixmatch: A holistic approach to semi-supervised learning.
\newblock In \emph{Advances in Neural Information Processing Systems}, pages
  5050--5060, 2019{\natexlab{b}}.

\bibitem[Blum and Mitchell(1998)]{blum1998combining}
Avrim Blum and Tom Mitchell.
\newblock Combining labeled and unlabeled data with co-training.
\newblock In \emph{Proceedings of Annual Conference on Computational Learning
  Theory}, pages 92--100, 1998.

\bibitem[Chapelle et~al.(2006)Chapelle, Scholkopf, and Zien]{chapelle2006semi}
Olivier Chapelle, Bernhard Scholkopf, and Alexander Zien.
\newblock \emph{Semi-Supervised Learning (1st edition)}.
\newblock Cambridge: The MIT Press, 2006.

\bibitem[Charles and Papailiopoulos(2018)]{charles2018stability}
Zachary Charles and Dimitris Papailiopoulos.
\newblock Stability and generalization of learning algorithms that converge to
  global optima.
\newblock In \emph{International Conference on Machine Learning}, pages
  745--754, 2018.

\bibitem[Chen et~al.(2011)Chen, Weinberger, and Chen]{chen2011automatic}
Minmin Chen, Kilian~Q Weinberger, and Yixin Chen.
\newblock Automatic feature decomposition for single view co-training.
\newblock In \emph{International Conference on Machine Learning}, pages
  953--960, 2011.

\bibitem[Chizat et~al.(2019)Chizat, Oyallon, and Bach]{chizat2019lazy}
Lenaic Chizat, Edouard Oyallon, and Francis Bach.
\newblock On lazy training in differentiable programming.
\newblock In \emph{Advances in Neural Information Processing Systems}, pages
  2937--2947, 2019.

\bibitem[Chung and Lu(2006)]{chung2006concentration}
Fan Chung and Linyuan Lu.
\newblock Concentration inequalities and martingale inequalities: a survey.
\newblock \emph{Internet Mathematics}, 3\penalty0 (1):\penalty0 79--127, 2006.

\bibitem[Coates et~al.(2011)Coates, Ng, and Lee]{coates2011analysis}
Adam Coates, Andrew Ng, and Honglak Lee.
\newblock An analysis of single-layer networks in unsupervised feature
  learning.
\newblock In \emph{International Conference on Artificial Intelligence and
  Statistics}, pages 215--223, 2011.

\bibitem[Cozman et~al.(2003)Cozman, Cohen, and Cirelo]{cozman2003semi}
Fabio~G Cozman, Ira Cohen, and Marcelo~C Cirelo.
\newblock Semi-supervised learning of mixture models.
\newblock In \emph{International Conference on Machine Learning}, pages
  99--106, 2003.

\bibitem[Cubuk et~al.(2019)Cubuk, Zoph, Mane, Vasudevan, and
  Le]{cubuk2019autoaugment}
Ekin~D Cubuk, Barret Zoph, Dandelion Mane, Vijay Vasudevan, and Quoc~V Le.
\newblock Autoaugment: Learning augmentation strategies from data.
\newblock In \emph{Proceedings of the IEEE Conference on Computer Vision and
  Pattern Recognition}, pages 113--123, 2019.

\bibitem[Cubuk et~al.(2020)Cubuk, Zoph, Shlens, and Le]{cubuk2020randaugment}
Ekin~D Cubuk, Barret Zoph, Jonathon Shlens, and Quoc~V Le.
\newblock Randaugment: Practical automated data augmentation with a reduced
  search space.
\newblock In \emph{Proceedings of the IEEE Conference on Computer Vision and
  Pattern Recognition Workshops}, pages 702--703, 2020.

\bibitem[Du et~al.(2019)Du, Lee, Li, Wang, and Zhai]{du2019gradient}
Simon Du, Jason Lee, Haochuan Li, Liwei Wang, and Xiyu Zhai.
\newblock Gradient descent finds global minima of deep neural networks.
\newblock In \emph{International Conference on Machine Learning}, pages
  1675--1685, 2019.

\bibitem[Flach(2012)]{flach2012machine}
Peter Flach.
\newblock \emph{Machine learning: the art and science of algorithms that make
  sense of data}.
\newblock Cambridge University Press, 2012.

\bibitem[Ghadimi and Lan(2013)]{ghadimi2013stochastic}
Saeed Ghadimi and Guanghui Lan.
\newblock Stochastic first-and zeroth-order methods for nonconvex stochastic
  programming.
\newblock \emph{SIAM Journal on Optimization}, 23\penalty0 (4):\penalty0
  2341--2368, 2013.

\bibitem[Ghadimi et~al.(2016)Ghadimi, Lan, and Zhang]{ghadimi2016mini}
Saeed Ghadimi, Guanghui Lan, and Hongchao Zhang.
\newblock Mini-batch stochastic approximation methods for nonconvex stochastic
  composite optimization.
\newblock \emph{Mathematical Programming}, 155\penalty0 (1-2):\penalty0
  267--305, 2016.

\bibitem[Goodfellow et~al.(2016)Goodfellow, Bengio, and
  Courville]{goodfellow2016deep}
Ian Goodfellow, Yoshua Bengio, and Aaron Courville.
\newblock \emph{Deep learning}.
\newblock MIT press, 2016.

\bibitem[Grandvalet and Bengio(2005)]{grandvalet2005semi}
Yves Grandvalet and Yoshua Bengio.
\newblock Semi-supervised learning by entropy minimization.
\newblock In \emph{Advances in Neural Information Processing Systems}, pages
  529--536, 2005.

\bibitem[Guo et~al.(2020)Guo, Zhang, Jiang, Li, and Zhou]{guosafe2020}
Lan-Zhe Guo, Zhen-Yu Zhang, Yuan Jiang, Yu-Feng Li, and Zhi-Hua Zhou.
\newblock Safe deep semi-supervised learning for unseen-class unlabeled data.
\newblock In \emph{International Conference on Machine Learning}, pages
  3897--3906, 2020.

\bibitem[Hady and Schwenker(2013)]{hady2013semi}
Mohamed Farouk~Abdel Hady and Friedhelm Schwenker.
\newblock Semi-supervised learning.
\newblock In \emph{Handbook on Neural Information Processing}, pages 215--239.
  Springer, 2013.

\bibitem[Hastie et~al.(2019)Hastie, Montanari, Rosset, and
  Tibshirani]{hastie2019surprises}
Trevor Hastie, Andrea Montanari, Saharon Rosset, and Ryan~J Tibshirani.
\newblock Surprises in high-dimensional ridgeless least squares interpolation.
\newblock \emph{arXiv preprint arXiv:1903.08560}, 2019.

\bibitem[Hataya and Nakayama(2019)]{hataya2019unifying}
Ryuichiro Hataya and Hideki Nakayama.
\newblock Unifying semi-supervised and robust learning by mixup.
\newblock 2019.

\bibitem[Karimi et~al.(2016)Karimi, Nutini, and Schmidt]{karimi2016linear}
Hamed Karimi, Julie Nutini, and Mark Schmidt.
\newblock Linear convergence of gradient and proximal-gradient methods under
  the polyak-{\l}ojasiewicz condition.
\newblock In \emph{Joint European Conference on Machine Learning and Knowledge
  Discovery in Databases}, pages 795--811. Springer, 2016.

\bibitem[Krijthe and Loog(2017)]{krijthe2017projected}
Jesse~H Krijthe and Marco Loog.
\newblock Projected estimators for robust semi-supervised classification.
\newblock \emph{Machine Learning}, 106\penalty0 (7):\penalty0 993--1008, 2017.

\bibitem[Krizhevsky and Hinton(2009)]{krizhevsky2009learning}
Alex Krizhevsky and Geoffrey Hinton.
\newblock Learning multiple layers of features from tiny images.
\newblock \emph{Master's thesis, Technical report, University of Tronto}, 2009.

\bibitem[Laine and Aila(2017)]{laine2017temporal}
Samuli Laine and Timo Aila.
\newblock Temporal ensembling for semi-supervised learning.
\newblock In \emph{International Conference on Learning Representations}, 2017.

\bibitem[Lee(2013)]{lee2013pseudo}
Dong-Hyun Lee.
\newblock Pseudo-label: The simple and efficient semi-supervised learning
  method for deep neural networks.
\newblock In \emph{In ICML Workshop on Challenges in Representation Learning},
  2013.

\bibitem[Li et~al.(2020)Li, Socher, and Hoi]{li2020dividemix}
Junnan Li, Richard Socher, and Steven~CH Hoi.
\newblock Dividemix: Learning with noisy labels as semi-supervised learning.
\newblock \emph{International Conference on Learning Representations}, 2020.

\bibitem[Li and Zhou(2011{\natexlab{a}})]{li2011improving}
Yu-Feng Li and Zhi-Hua Zhou.
\newblock Improving semi-supervised support vector machines through unlabeled
  instances selection.
\newblock In \emph{Proceedings of the AAAI Conference on Artificial
  Intelligence}, pages 386--391, 2011{\natexlab{a}}.

\bibitem[Li and Zhou(2011{\natexlab{b}})]{li2011towards}
Yu-Feng Li and Zhi-Hua Zhou.
\newblock Towards making unlabeled data never hurt.
\newblock In \emph{International Conference on Machine Learning}, pages
  1081--1088, 2011{\natexlab{b}}.

\bibitem[Li and Zhou(2015)]{li2014towards}
Yu-Feng Li and Zhi-Hua Zhou.
\newblock Towards making unlabeled data never hurt.
\newblock \emph{IEEE Transactions on Pattern Analysis and Machine
  Intelligence}, 37\penalty0 (1):\penalty0 175--188, 2015.

\bibitem[Li et~al.(2017)Li, Zha, and Zhou]{li2017learning}
Yu-Feng Li, Han-Wen Zha, and Zhi-Hua Zhou.
\newblock Learning safe prediction for semi-supervised regression.
\newblock In \emph{Proceedings of the Thirty-First AAAI Conference on
  Artificial Intelligence}, pages 2217--2223, 2017.

\bibitem[Li et~al.(2021)Li, Guo, and Zhou]{li2019towards}
Yu-Feng Li, Lan-Zhe Guo, and Zhi-Hua Zhou.
\newblock Towards safe weakly supervised learning.
\newblock \emph{IEEE Transactions on Pattern Analysis and Machine
  Intelligence}, 43\penalty0 (1):\penalty0 334--346, 2021.

\bibitem[Li and Li(2018)]{li2018simple}
Zhize Li and Jian Li.
\newblock A simple proximal stochastic gradient method for nonsmooth nonconvex
  optimization.
\newblock In \emph{Advances in Neural Information Processing Systems}, pages
  5564--5574, 2018.

\bibitem[Loog(2015)]{loog2015contrastive}
Marco Loog.
\newblock Contrastive pessimistic likelihood estimation for semi-supervised
  classification.
\newblock \emph{IEEE Transactions on Pattern Analysis and Machine
  Intelligence}, 38\penalty0 (3):\penalty0 462--475, 2015.

\bibitem[Loshchilov and Hutter(2017)]{loshchilov2016sgdr}
Ilya Loshchilov and Frank Hutter.
\newblock Sgdr: Stochastic gradient descent with warm restarts.
\newblock \emph{International Conference on Learning Representations}, 2017.

\bibitem[Mammen et~al.(1999)Mammen, Tsybakov, et~al.]{mammen1999smooth}
Enno Mammen, Alexandre~B Tsybakov, et~al.
\newblock Smooth discrimination analysis.
\newblock \emph{The Annals of Statistics}, 27\penalty0 (6):\penalty0
  1808--1829, 1999.

\bibitem[McLachlan(1975)]{mclachlan1975iterative}
Geoffrey~J McLachlan.
\newblock Iterative reclassification procedure for constructing an
  asymptotically optimal rule of allocation in discriminant analysis.
\newblock \emph{Journal of the American Statistical Association}, 70\penalty0
  (350):\penalty0 365--369, 1975.

\bibitem[Mey and Loog(2019)]{mey2019improvability}
Alexander Mey and Marco Loog.
\newblock Improvability through semi-supervised learning: a survey of
  theoretical results.
\newblock \emph{arXiv preprint arXiv:1908.09574}, 2019.

\bibitem[Misra et~al.(2015)Misra, Shrivastava, and Hebert]{misra2015watch}
Ishan Misra, Abhinav Shrivastava, and Martial Hebert.
\newblock Watch and learn: Semi-supervised learning for object detectors from
  video.
\newblock In \emph{Proceedings of the IEEE Conference on Computer Vision and
  Pattern Recognition}, pages 3593--3602, 2015.

\bibitem[Miyato et~al.(2018)Miyato, Maeda, Koyama, and
  Ishii]{miyato2018virtual}
Takeru Miyato, Shin-ichi Maeda, Masanori Koyama, and Shin Ishii.
\newblock Virtual adversarial training: a regularization method for supervised
  and semi-supervised learning.
\newblock \emph{IEEE Transactions on Pattern Analysis and Machine
  Intelligence}, 41\penalty0 (8):\penalty0 1979--1993, 2018.

\bibitem[Netzer et~al.(2011)Netzer, Wang, Coates, Bissacco, Wu, and
  Ng]{Netzer201137648}
Yuval Netzer, Tao Wang, Adam Coates, Alessandro Bissacco, Bo~Wu, and Andrew~Y.
  Ng.
\newblock Reading digits in natural images with unsupervised feature learning.
\newblock 2011.

\bibitem[Oliver et~al.(2018)Oliver, Odena, Raffel, Cubuk, and
  Goodfellow]{oliver2018realistic}
Avital Oliver, Augustus Odena, Colin~A Raffel, Ekin~Dogus Cubuk, and Ian
  Goodfellow.
\newblock Realistic evaluation of deep semi-supervised learning algorithms.
\newblock In \emph{Advances in Neural Information Processing Systems}, pages
  3235--3246, 2018.

\bibitem[Polyak(1963)]{polyak1963gradient}
Boris~Teodorovich Polyak.
\newblock Gradient methods for minimizing functionals.
\newblock \emph{Zhurnal Vychislitel'noi Matematiki i Matematicheskoi Fiziki},
  3\penalty0 (4):\penalty0 643--653, 1963.

\bibitem[Rasmus et~al.(2015)Rasmus, Valpola, Honkala, Berglund, and
  Raiko]{rasmus2015semi}
Antti Rasmus, Harri Valpola, Mikko Honkala, Mathias Berglund, and Tapani Raiko.
\newblock Semi-supervised learning with ladder networks.
\newblock In \emph{Neural Information Processing Systems}, pages 3546--3554,
  2015.

\bibitem[Ren et~al.(2020)Ren, Yeh, and Schwing]{ren2020not}
Zhongzheng Ren, Raymond Yeh, and Alexander Schwing.
\newblock Not all unlabeled data are equal: learning to weight data in
  semi-supervised learning.
\newblock \emph{Advances in Neural Information Processing Systems}, 33, 2020.

\bibitem[Rosenberg et~al.(2005)Rosenberg, Hebert, and
  Schneiderman]{rosenberg2005semi}
Chuck Rosenberg, Martial Hebert, and Henry Schneiderman.
\newblock Semi-supervised self-training of object detection models.
\newblock In \emph{Proceedings of the 7th IEEE Workshops on Application of
  Computer Vision}, pages 29--36, 2005.

\bibitem[Sajjadi et~al.(2016)Sajjadi, Javanmardi, and
  Tasdizen]{sajjadi2016regularization}
Mehdi Sajjadi, Mehran Javanmardi, and Tolga Tasdizen.
\newblock Regularization with stochastic transformations and perturbations for
  deep semi-supervised learning.
\newblock In \emph{Advances in Neural Information Processing Systems}, pages
  1163--1171, 2016.

\bibitem[Schmidhuber(2015)]{schmidhuber2015deep}
J{\"u}rgen Schmidhuber.
\newblock Deep learning in neural networks: An overview.
\newblock \emph{Neural networks}, 61:\penalty0 85--117, 2015.

\bibitem[Sindhwani and Rosenberg(2008)]{sindhwani2008rkhs}
Vikas Sindhwani and David~S Rosenberg.
\newblock An rkhs for multi-view learning and manifold co-regularization.
\newblock In \emph{International Conference on Machine Learning}, pages
  976--983, 2008.

\bibitem[Singh et~al.(2009)Singh, Nowak, and Zhu]{singh2009unlabeled}
Aarti Singh, Robert Nowak, and Jerry Zhu.
\newblock Unlabeled data: Now it helps, now it doesn't.
\newblock In \emph{Advances in Neural Information Processing Systems}, pages
  1513--1520, 2009.

\bibitem[Sohn et~al.(2020)Sohn, Berthelot, Li, Zhang, Carlini, Cubuk, Kurakin,
  Zhang, and Raffel]{sohn2020fixmatch}
Kihyuk Sohn, David Berthelot, Chun-Liang Li, Zizhao Zhang, Nicholas Carlini,
  Ekin~D Cubuk, Alex Kurakin, Han Zhang, and Colin Raffel.
\newblock Fixmatch: Simplifying semi-supervised learning with consistency and
  confidence.
\newblock In \emph{Advances in Neural Information Processing Systems}, pages
  596--608, 2020.

\bibitem[Tarvainen and Valpola(2017)]{tarvainen2017mean}
Antti Tarvainen and Harri Valpola.
\newblock Mean teachers are better role models: Weight-averaged consistency
  targets improve semi-supervised deep learning results.
\newblock In \emph{Advances in Neural Information Processing Systems}, pages
  1195--1204, 2017.

\bibitem[Tsybakov(2004)]{tsybakov2004optimal}
Alexander~B Tsybakov.
\newblock Optimal aggregation of classifiers in statistical learning.
\newblock \emph{The Annals of Statistics}, 32\penalty0 (1):\penalty0 135--166,
  2004.

\bibitem[Turian et~al.(2010)Turian, Ratinov, and Bengio]{turian2010word}
Joseph Turian, Lev Ratinov, and Yoshua Bengio.
\newblock Word representations: a simple and general method for semi-supervised
  learning.
\newblock In \emph{Proceedings of Annual Meeting of the Association for
  Computational Linguistics}, pages 384--394, 2010.

\bibitem[Van~Engelen and Hoos(2020)]{van2020survey}
Jesper~E Van~Engelen and Holger~H Hoos.
\newblock A survey on semi-supervised learning.
\newblock \emph{Machine Learning}, 109\penalty0 (2):\penalty0 373--440, 2020.

\bibitem[Wang et~al.(2008)Wang, Luo, and Zeng]{wang2008random}
Jiao Wang, Si-wei Luo, and Xian-hua Zeng.
\newblock A random subspace method for co-training.
\newblock In \emph{Proceedings of IEEE International Joint Conference on Neural
  Networks}, pages 195--200, 2008.

\bibitem[Wang and Zhou(2010)]{wang2010new}
Wei Wang and Zhi-Hua Zhou.
\newblock A new analysis of co-training.
\newblock In \emph{International Conference on Machine Learning}, pages
  1135--1142, 2010.

\bibitem[Wang et~al.(2019)Wang, Ji, Zhou, Liang, and
  Tarokh]{wang2019spiderboost}
Zhe Wang, Kaiyi Ji, Yi~Zhou, Yingbin Liang, and Vahid Tarokh.
\newblock Spiderboost and momentum: Faster variance reduction algorithms.
\newblock In \emph{Advances in Neural Information Processing Systems}, pages
  2403--2413, 2019.

\bibitem[Xie et~al.(2020{\natexlab{a}})Xie, Dai, Hovy, Luong, and
  Le]{xie2020unsupervised}
Qizhe Xie, Zihang Dai, Eduard Hovy, Thang Luong, and Quoc Le.
\newblock Unsupervised data augmentation for consistency training.
\newblock In \emph{Advances in Neural Information Processing Systems}, pages
  6256--6268, 2020{\natexlab{a}}.

\bibitem[Xie et~al.(2020{\natexlab{b}})Xie, Luong, Hovy, and Le]{xie2020self}
Qizhe Xie, Minh-Thang Luong, Eduard Hovy, and Quoc~V Le.
\newblock Self-training with noisy student improves imagenet classification.
\newblock In \emph{Proceedings of the IEEE Conference on Computer Vision and
  Pattern Recognition}, pages 10687--10698, 2020{\natexlab{b}}.

\bibitem[Yarowsky(1995)]{yarowsky1995unsupervised}
David Yarowsky.
\newblock Unsupervised word sense disambiguation rivaling supervised methods.
\newblock In \emph{Proceedings of Annual Meeting of the Association for
  Computational Linguistics}, pages 189--196, 1995.

\bibitem[Yu et~al.(2010)Yu, Varadarajan, Deng, and Acero]{yu2010active}
Dong Yu, Balakrishnan Varadarajan, Li~Deng, and Alex Acero.
\newblock Active learning and semi-supervised learning for speech recognition:
  A unified framework using the global entropy reduction maximization
  criterion.
\newblock \emph{Computer Speech \& Language}, 24\penalty0 (3):\penalty0
  433--444, 2010.

\bibitem[Yu et~al.(2008)Yu, Krishnapuram, Steck, Rao, and
  Rosales]{yu2008bayesian}
Shipeng Yu, Balaji Krishnapuram, Harald Steck, R~Bharat Rao, and R{\'o}mer
  Rosales.
\newblock Bayesian co-training.
\newblock In \emph{Advances in Neural Information Processing Systems}, pages
  1665--1672, 2008.

\bibitem[Yuan et~al.(2019)Yuan, Yan, Jin, and Yang]{yuan2019stagewise}
Zhuoning Yuan, Yan Yan, Rong Jin, and Tianbao Yang.
\newblock Stagewise training accelerates convergence of testing error over sgd.
\newblock In \emph{Advances in Neural Information Processing Systems}, pages
  2604--2614, 2019.

\bibitem[Yun et~al.(2019)Yun, Sra, and Jadbabaie]{yun2019small}
Chulhee Yun, Suvrit Sra, and Ali Jadbabaie.
\newblock Small relu networks are powerful memorizers: a tight analysis of
  memorization capacity.
\newblock In \emph{Advances in Neural Information Processing Systems}, pages
  15558--15569, 2019.

\bibitem[Zagoruyko and Komodakis(2016)]{zagoruyko2016wide}
Sergey Zagoruyko and Nikos Komodakis.
\newblock Wide residual networks.
\newblock \emph{British Machine Vision Conference}, 2016.

\bibitem[Zhang et~al.(2017)Zhang, Bengio, Hardt, Recht, and
  Vinyals]{zhang2016understanding}
Chiyuan Zhang, Samy Bengio, Moritz Hardt, Benjamin Recht, and Oriol Vinyals.
\newblock Understanding deep learning requires rethinking generalization.
\newblock \emph{International Conference on Learning Representations}, 2017.

\bibitem[Zhou and Li(2005)]{zhou2005tri}
Zhi-Hua Zhou and Ming Li.
\newblock Tri-training: Exploiting unlabeled data using three classifiers.
\newblock \emph{IEEE Transactions on knowledge and Data Engineering},
  17\penalty0 (11):\penalty0 1529--1541, 2005.

\bibitem[Zhu and Goldberg(2009)]{zhu2009introduction}
Xiaojin Zhu and Andrew~B Goldberg.
\newblock Introduction to semi-supervised learning.
\newblock \emph{Synthesis Lectures on Artificial Intelligence and Machine
  Learning}, 3\penalty0 (1):\penalty0 1--130, 2009.

\bibitem[Zhu(2005)]{zhu2005semi}
Xiaojin~Jerry Zhu.
\newblock Semi-supervised learning literature survey.
\newblock Technical report, Technical Report, University of Wisconsin-Madison
  Department of Computer Sciences, 2005.

\end{thebibliography}

\newpage
\appendix
\section{Proof of Theorem~\ref{thm:main}}\label{sec:proof}
In this section, we present the proof of our main theoretical result. To this end, we divide the analysis into two part, with the first part devoted to examining the properties of $\w_1$ learned in the first step and the second part devoted to the convergence for the iterations. As the setting of SSL, we assume that the number of unlabeled data is large, i.e., $N_u$ is sufficiently large. 
\subsection{Properties of Solution $\w_1$}
We first give the property of $\w_1$ in the following lemma, whose proof can be found in the Appendix. 
\begin{lemma}\label{lem:0}
Given $\delta\in(0,1)$, run the Warm-up Stage of Algorithm~\ref{alg:dash} with $\eta_0\le \frac{1}{L}$, $T_0 = \frac{\log(2F(\w_0)/a)}{\log(1/(1-\eta_0\mu))}$ and $m_0 =  \frac{4 G^2}{\delta \mu a}$, then with a probability $1 - 5\delta$ we have
    $F(\w_1) \leq  \widehat\rho$. 
\end{lemma}
\begin{proof}
Since the objective function $F(\u)$ has a Lipchitz continuous gradient in Assumption~\ref{ass:2} (ii), we have
\begin{align}\label{lem:0:ineq:0}
\nonumber&F(\u_{t+1}) - F(\u_t) \\
\nonumber\le& \langle \nabla F(\u_t), \u_{t+1} - \u_t \rangle + \frac{L}{2}\|\u_{t+1} - \u_t\|^2\\
\nonumber\overset{(a)}{=}& \frac{\eta_0}{2}\|\nabla F(\u_t) - \tilde\g_t\|^2- \frac{\eta_0}{2}\left(\|\nabla F(\u_t)\|^2 + \left(1 - \eta_0 L\right)\|\tilde\g_t\|^2 \right) \\
\overset{(b)}{\leq}& \frac{\eta_0}{2}\|\nabla F(\u_t) - \tilde\g_t\|^2 - \eta_0\mu F(\u_t),
\end{align}
where (a) follows the update of $\u_{t+1} =\u_t - \eta_0 \tilde\g_t$;  (b) follows from Assumption~\ref{ass:3} and $\eta_0 L \le 1$. 
Since $\E[\tilde\g_t] = \nabla F(\u_t)$ and
\begin{align}\label{lem:0:ineq:1}
\nonumber &\E\left[\|\tilde\g_t-\nabla F(\u_t)\|^2\right] \\
\nonumber =&  \frac{1}{m_0^2}\sum_{i=1}^{m_0}\E\left[\|\nabla f(\u_{t};\xi_i^{t})-\nabla F(\u_t)\|^2\right]\\
\le & \frac{4G^2}{m_0},~~(\text{Assumption~\ref{ass:2} (i)})
\end{align}
using concentration inequality in Lemma 4 of \citep{ghadimi2016mini}, 
we have with a probability $1 - 5\delta$,
\begin{align}\label{lem:0:ineq:2}
\left\|\tilde\g_t - \nabla F(\u_t)\right\|^2 \leq \frac{4G^2}{\delta m_0}.
\end{align}
Using the above bound $\|\tilde\g_t - \nabla F(\u_t)\|$, we can further bound $F(\u_T)$ by using (\ref{lem:0:ineq:0}) and (\ref{lem:0:ineq:1}) as
\begin{align*}
F(\u_{T_0}) \leq & (1 - \eta_0\mu)F(\u_{T_0-1}) + \frac{2\eta_0 G^2}{\delta m_0}\\
\leq & (1 - \eta_0\mu)^{T_0} F(\u_0) + \frac{2\eta_0 G^2}{\delta m_0} \sum_{i=0}^{T_0-1} (1-\eta_0\mu)^i\\
  \leq & (1 - \eta_0\mu)^{T_0} F(\u_0) + \frac{2 G^2}{\delta m_0\mu},
\end{align*}
which implies 
\begin{align}\label{lem:0:ineq:3}
 F(\w_1)  \leq & (1 - \eta_0\mu)^{T_0} F(\w_0) + \frac{2 G^2}{\delta m_0\mu}.
\end{align}
Be selecting $T_0 \ge \frac{\log(2F(\w_0)/a)}{\log(1/(1-\eta_0\mu))}$ and $m_0 \ge \frac{4 G^2}{\delta \mu a}$, we have
\begin{align}\label{lem:0:ineq:4}
    F(\w_1) \leq a.
\end{align}
Therefore the condition in (\ref{eqn:condition}) is applicable in this case. On the other hand, by the definition of $\widehat\rho$ in (\ref{def:rho:hat}), we know, with a probability $1 - 5\delta$, that
\begin{align}\label{lem:0:ineq:7}
F(\w_1) \le \widehat{\rho}.
\end{align}
\end{proof}

\subsection{Analysis of Iterative Algorithm}
The key to our analysis is to show that for each iteration $t$, with a high probability, we have $F(\w_t) \leq \widehat{\rho}\gamma^{-(t-1)}$. We will prove this statement by induction.  
Before we carry out our analysis, we define a few important constants
\begin{align}
\beta =& \max\left(\sqrt{\frac{\log(2/\delta)}{2q^2m}}, \sqrt{\frac{\log(2/\delta)}{2(1-q)^2m}}\right),\label{eqn:beta}\\
\alpha  = & \sqrt{\frac{\log(2/\delta)}{qm(1-C^{-1})^2}}
\ge \sqrt{\frac{\log(2/\delta)}{2qm(1-\beta)(1-C^{-1})^2}},\label{eqn:alpha}\\
a_0 =& (1 - C^{-1})(1- \beta)(1-\alpha)q, \label{eqn:A0}\\
b_0 =& 2\left((1-q)(1+\beta)  b \widehat\rho^\theta+ \log(1/\delta)\right), \label{eqn:B0}\\
b_1 =& \frac{b_0}{a_0}. \label{eqn:AB}
\end{align}
When $t = 1$, we have $F(\w_1) \leq \widehat{\rho}$ according to Lemma~\ref{lem:0}. At each iteration $t$, given the solution $\w_t$, according to our inductive assumption, with a probability $1 - (4t+1)\delta$, we have
\begin{align}\label{inq:induction}
F(\w_t) \le \widehat\rho \gamma^{-(t-1)},
\end{align}
For the $n_t = m\gamma^{t-1}$ training examples sampled from $\D_u$, we divide it into two sets for the analysis use only, i.e. set $\A_t$ that includes examples sampled from $\P$ and set $\B_t$ that includes examples sampled from $\Q$. We furthermore denote by $\A_t^{\rho}$ and $\B_t^{\rho}$ the subset of examples in $\A_t$ and $\B_t$ whose loss is smaller than the given threshold $\rho_t$, i.e.
\begin{align}
\A_t^{\rho} = &\left\{\xi \in \A_t: f(\w_t;\xi) \leq \rho_t \right\}, \\
\B_t^{\rho} = &\left\{\xi \in \B_t:  f(\w_t;\xi) \leq \rho_t \right\},
\end{align}
where $\rho_t = C\widehat{\rho}\gamma^{-(t-1)}$ with $C > 1$. Evidently, the samples used for computing $g_t$ is the union of $\A_t^{\rho}$ and $\B_t^{\rho}$. The following result bounds the size of $\A_t^{\rho}$ and $\B_t^{\rho}$.
With a probability $1 - 4\delta$, we have
\begin{align}\label{prop:1}
    |\A_t^{\rho}| \geq a_0m\gamma^{(t-1)}, \quad |\B_t^\rho| \leq b_0m,
\end{align}
where $a_0$ and $b_0$ are defined in (\ref{eqn:A0}) and (\ref{eqn:B0}).
Using the Hoeffding's inequality, 
we have, with a probability $1 - 2\delta$,
\begin{align}
    |\A_t| \geq & qn_t\left( 1 - \sqrt{\frac{\log(2/\delta)}{2q^2n_t}}\right)  \overset{(\ref{eqn:beta})}{\geq} qn_t\left(1 - \beta\right), \label{eqn:0:1}\\
   |\B_t| \leq & (1-q)n_t\left(1 + \sqrt{\frac{\log(2/\delta)}{2(1 - q)^2n_t}}\right) 
    \overset{(\ref{eqn:beta})}{\leq}  (1-q)n_t\left(1 + \beta\right). \label{eqn:0:2}
\end{align}
Furthermore, using Markov inequality, 
we have
\begin{align}\label{eqn:1}
\nonumber  \text{Pr}_{\xi\sim \P}(f(\w_t;\xi) \le \rho_t) = &1- \text{Pr}_{\xi\sim \P}(f(\w_t;\xi) \ge \rho_t) \\
\nonumber   \overset{\text{Markov inequality} 
}{\ge} &1- \frac{\E_{\xi\sim \P}[f(\w_t;\xi)]}{\rho_t}\\
\nonumber   = &1- \frac{F(\w_t)}{C\widehat\rho \gamma^{-(t-1)}}\\
\overset{(a)}{\ge}&1- \frac{\widehat\rho \gamma^{-(t-1)}}{C\widehat\rho \gamma^{-(t-1)}}
= 1 - C^{-1},
\end{align}
where (a) uses the fact that $F(\w_t)\le \widehat\rho \gamma^{-(t-1)}$. Using the property in (\ref{eqn:1}), we can bound the size of $\A_t^{\rho}$ by Hoeffding's inequality, 
i.e., with a probability $1 - 2\delta$, we have
\begin{align} \label{eqn:2}
|\A_t^{\rho}| \geq \left(1 - C^{-1}\right)|\A_t|\left( 1 - \sqrt{\frac{\log(2/\delta)}{2\left(1 - C^{-1}\right)^2|\A_t|}}\right) 
 \overset{(\ref{eqn:beta})(\ref{eqn:alpha})}{\geq}   \left(1 - C^{-1}\right)(1-\beta)(1-\alpha)qn_t :=a_0m\gamma^{t-1},
\end{align}
where $a_0$ is defined in (\ref{eqn:A0}). 
Using the inequality in (\ref{eqn:condition}) we know
\begin{align}\label{eqn:3}
\nonumber  \text{Pr}_{\xi\sim \Q}(f(\w_t;\xi) \le \rho_t) = &\text{Pr}_{\xi\sim \Q}(f(\w_t;\xi) \le C\widehat\rho \gamma^{-(t-1)}) \\
\nonumber    \overset{(a)}{\le} &\text{Pr}_{\xi\sim \Q}(f(\w_t;\xi) \le \widehat\rho \gamma^{-(t-1)}) \\
\nonumber    \overset{(b)}{\le} &\text{Pr}_{\xi\sim \Q}(f(\w_t;\xi) \le F(\w_t)) \\
\nonumber    = &\E_{\xi\sim \Q}\left[I\left(f(\w_t;\xi) \leq F(\w_t)\right)\right]\\
\overset{(\ref{eqn:condition})}{\le} & b A^\theta F(\w_t)
\overset{(c)}{\le} b (\widehat\rho \gamma^{-(t-1)})^\theta, 
\end{align}
where (a) is due to $C>1$; (b) and (c) are due to $F(\w_t)\le \widehat\rho \gamma^{-(t-1)}$.
Using the property in (\ref{eqn:3}), we can bound the expectation of $|\B_t^\rho|$ given $|\B_t|$, i.e.,
\begin{align}\label{eqn:4}
\nonumber    \E_{\xi\sim \Q}[|\B_t^\rho|] =& \E_{\xi_i\sim \Q}\left[\sum_{i=1}^{|\B_t|}I\left( f(\w_t;\xi_i) \le \rho_t\right) \right]\\
\nonumber   = &\sum_{i=1}^{|\B_t|}\E_{\xi_i\sim \Q}\left[I\left( f(\w_t;\xi_i) \le \rho_t\right) \right]\\
\nonumber   = &\sum_{i=1}^{|\B_t|}\text{Pr}_{\xi_i\sim \Q}(f(\w_t;\xi_i) \le \rho_t)\\
\overset{(\ref{eqn:3})}{\le}& |\B_t| b (\widehat\rho \gamma^{-(t-1)})^\theta.
\end{align}
We can bound the size of $\B_t^{\rho}$ by a concentration inequality in Theorem 8 of~\citep{chung2006concentration}, 
with a probability $1 - 2\delta$,
\begin{align}\label{eqn:5}
\nonumber |\B_t^{\rho}| \leq&\E_{\xi\sim \Q}[|\B_t^\rho|]+ \frac{1}{3}\log(1/\delta)+ \sqrt{\frac{1}{9}\log^2(1/\delta) + 2\log(1/\delta) \E_{\xi\sim \Q}[|\B_t^\rho|]} \\
\nonumber \overset{(\ref{eqn:4})}{\le}&|\B_t| b (\widehat\rho \gamma^{-(t-1)})^\theta+ \frac{1}{3}\log(1/\delta)+ \sqrt{\frac{1}{9}\log^2(1/\delta) + 2\log(1/\delta) |\B_t| b (\widehat\rho \gamma^{-(t-1)})^\theta}  \\
\nonumber \le&2|\B_t| b (\widehat\rho \gamma^{-(t-1)})^\theta + 2\log(1/\delta) \\
\nonumber \overset{(\ref{eqn:0:2})}{\le}&2 (1-q)m \gamma^{t-1} (1+\beta)  b (\widehat\rho \gamma^{-(t-1)})^\theta+ 2\log(1/\delta) \\
\leq &2\left((1-q)(1+\beta)  b \widehat\rho^\theta+ \log(1/\delta)\right) m := b_0m,
\end{align}
where $b_0$ is defined in (\ref{eqn:B0}). Therefore, we complete the proof of (\ref{prop:1}).

As indicated in (\ref{eqn:2}) and (\ref{eqn:5}), $|\A_t^{\rho}|$ increases exponentially over iteration while $|\B_t^{\rho}|$ remains upper bounded by a constant. It implies that our dynamically adjusted threshold help us select more and more examples from the unlabeled data that are relevant to the labeled data. In the same time, the number of mistakes we made in the selection process remain to be at most a constant. As a result, our optimization is able to make significant progress by using the selected examples from the unlabeled data. Below we will show that with a high probability, $F(\w_{t+1}) \leq \widehat{\rho}\gamma^{-t}$, the key step of the inductive analysis. 
Finally, we will prove that with a probability $1 -(4t+1)\delta$, we have
\begin{align*}
    F(\w_{t+1}) \leq \widehat{\rho}\gamma^{-t}.
\end{align*}

To this end, using the notation of $\A_t^{\rho}$ and $\B_t^{\rho}$, we can rewrite $\g_t$ as
\begin{align*}
\g_t = (1 - b_t) \g_t^a + b_t \g_t^b,
\end{align*}
where $\g_t^a = \frac{1}{|\A_t^{\rho}|}\sum_{\xi\in \A_t^{\rho}} \nabla f(\w_t;\xi), \quad \g_t^b = \frac{1}{|\B_t^{\rho}|}\sum_{\xi\in \B_t^{\rho}} \nabla f(\w_t;\xi)$, and $b_t $ is the proportion of samples from $\B_t^{\rho}$ that
\begin{align}\label{thm:2:inq:1}
b_t = \frac{|\B_t^{\rho}|}{|\A_t^{\rho}| + |\B_t^{\rho}|} \le \frac{|\B_t^{\rho}|}{|\A_t^{\rho}|} 
\leq \frac{b_0}{a_0}\gamma^{-(t-1)} = b_1\gamma^{-(t-1)}.
\end{align}
Following the classical analysis of non-convex optimization, since $F(\w)$ is $L$-smooth by Assumption~\ref{ass:2}~(ii), we have
\begin{align*}
&F(\w_{t+1}) - F(\w_t) \\
\le& \langle \nabla F(\w_t), \w_{t+1} - \w_t \rangle + \frac{L}{2}\|\w_{t+1} - \w_t\|^2\\
\overset{(a)}{=}& \frac{\eta}{2}\|\nabla F(\w_t) - \g_t\|^2- \frac{\eta}{2}\left(\|\nabla F(\w_t)\|^2 + \left(1 - \eta L\right)\|\g_t\|^2 \right) \\
\overset{(b)}{\leq} & \frac{\eta}{2}\left((1 - b_t)\|\nabla F(\w_t) - \g_t^a\|^2 + b_t\|\nabla F(\w_t) - \g_t^b\|^2\right) - \frac{\eta}{2}\left(\|\nabla F(\w_t)\|^2 + \left(1 - \eta L\right)\|\g_t\|^2 \right) \\
\overset{(c)}{\leq}& \frac{\eta}{2}\left((1 - b_t)\|\nabla F(\w_t) - \g_t^a\|^2 + 4b_t G^2\right) - \eta\mu F(\w_t),
\end{align*}
where (a) follows the update of $\w_{t+1} = \w_t - \eta \g_t$;  (b) is due to the definition of $\g_t$ and the convexity of $\|\cdot\|^2$; (c) follows from Assumption~\ref{ass:3}, Assumption~\ref{ass:2} (i), and $\eta L \le 1$. Since $\E_{\xi\sim\P}[\nabla f(\w;\xi)] = \nabla F(\w)$ and
\begin{align*}
&\E_{\xi\sim\P}\left[\left\|\g_t^a - \nabla F(\w_t)\right\|^2 \right] \\
=& \E_{\xi\sim\P}\left[ \left\|\frac{1}{|\A_t^{\rho}|}\sum_{\xi \in \A_t^{\rho}} \nabla f(\w_t;\xi) - \nabla F(\w_t) \right\|^2\right]\\
=& \frac{1}{|\A_t^{\rho}|^2}\sum_{\xi \in \A_t^{\rho}} \E_{\xi\sim\P}\left[ \left\| \nabla f(\w_t;\xi) - \nabla F(\w_t) \right\|^2\right]\\
\leq& \frac{4G^2}{|\A_t^{\rho}|},~~(\text{Assumption~\ref{ass:2} (i)})
\end{align*}
using concentration inequality in Lemma 4 of \citep{ghadimi2016mini}, 
we have with a probability $1 - 5\delta$,
\begin{align*}
\left\|\g_t^a - \nabla F(\w_t)\right\|^2 \leq \frac{4G^2}{\delta|\A_t^{\rho}|}  \overset{(\ref{prop:1})}{\le} \frac{4G^2}{\delta a_0m\gamma^{t-1}}.
\end{align*}
Using the above bound $\|\g_t^a - \nabla F(\w)\|$, we can further bound $F(\w_{t+1}) - F(\w_t)$ as
\begin{align}\label{thm:2:inq:2}
\nonumber & F(\w_{t+1}) - F(\w_t) \\
\nonumber \le &  \frac{\eta}{2}\left((1 - b_t)\frac{4G^2}{\delta a_0m\gamma^{t-1}} + 4b_t G^2\right) - \eta\mu F(\w_t)\\
\nonumber \overset{(\ref{thm:2:inq:1})}{\le} &  \frac{\eta}{2}\left(\frac{4G^2}{\delta a_0m\gamma^{t-1}} + 4G^2b_1\gamma^{-(t-1)}\right) - \eta\mu F(\w_t)\\
= & 2\eta G^2 \left(\frac{1}{\delta a_0m} + b_1\right) \gamma^{-(t-1)} - \eta\mu F(\w_t)
\end{align}
where $b_1$ is defined in (\ref{eqn:AB}). Hence, we have
\begin{align}\label{thm:2:inq:3}
F(\w_{t+1}) \leq & (1 - \eta\mu)F(\w_t) + 2\eta G^2 \left(\frac{1}{\delta a_0m} + b_1\right) \gamma^{-(t-1)}
\end{align}
Let select 
$\gamma = \frac{1}{1 - \eta\mu/2}$,
we know by using (\ref{inq:induction}) the inequality (\ref{thm:2:inq:3}) will become
\begin{align*}
\nonumber  F(\w_{t+1}) \leq & \gamma(1 - \eta\mu)\widehat\rho \gamma^{-t} + 2\eta \gamma G^2 \left(\frac{1}{\delta a_0m} + b_1\right) \gamma^{-t}\\
\nonumber=&\left(1 - \frac{\eta\mu/2}{1-\eta\mu/2}\right) \widehat\rho \gamma^{-t} \\
&+ \frac{\eta\mu/2}{1-\eta\mu/2} \frac{4G^2}{\mu} \left(\frac{1}{\delta a_0m} + b_1\right) \gamma^{-t}
\end{align*}
By the setting of $ \widehat{\rho}$ 
we have
\begin{align}
F(\w_{t+1}) \leq \widehat{\rho}\gamma^{-t}.
\end{align}
Therefore, we complete the proof by induction.
\end{document}